\documentclass{article}

\usepackage{arxiv}

\usepackage[utf8]{inputenc} 
\usepackage[T1]{fontenc}    
\usepackage{hyperref}       
\usepackage{url}            
\usepackage{booktabs}       
\usepackage{amsfonts}       
\usepackage{nicefrac}       
\usepackage{microtype}      
\usepackage{lipsum}
\usepackage{graphicx}
\graphicspath{ {./images/} }

\usepackage[round]{natbib}

\usepackage{amsmath,amsfonts,bm}

\def\eqref#1{equation~\ref{#1}}

\def\1{\bm{1}}

\def\vmu{{\bm{\mu}}}
\def\vtheta{{\bm{\theta}}}

\def\vc{{\bm{c}}}

\def\ve{{\bm{e}}}

\def\vw{{\bm{w}}}
\def\vx{{\bm{x}}}
\def\vy{{\bm{y}}}

\def\evw{{w}}

\def\mC{{\bm{C}}}

\def\mL{{\bm{L}}}

\DeclareMathAlphabet{\mathsfit}{\encodingdefault}{\sfdefault}{m}{sl}
\SetMathAlphabet{\mathsfit}{bold}{\encodingdefault}{\sfdefault}{bx}{n}

\def\sA{{\mathbb{A}}}

\def\sL{{\mathbb{L}}}

\def\sN{{\mathbb{N}}}

\def\sS{{\mathbb{S}}}
\def\sT{{\mathbb{T}}}

\def\sY{{\mathbb{Y}}}

\newcommand{\E}{\mathbb{E}}

\DeclareMathOperator*{\argmax}{arg\,max}
\DeclareMathOperator*{\argmin}{arg\,min}

\usepackage{hyperref}
\usepackage{url}

\usepackage{amsmath,amssymb,amsfonts,amsthm}
\usepackage{algorithmic}
\usepackage{graphicx}
\usepackage{bbm}
\usepackage{textcomp}
\usepackage{xcolor}
\usepackage{esvect}
\usepackage{mathtools}
\usepackage{physics}
\usepackage{stmaryrd}
\usepackage{subcaption}
\usepackage{enumitem}
\usepackage{circledsteps}

\usepackage{tikz}
\usetikzlibrary{3d, shapes, arrows.meta, positioning}

\def\BibTeX{{\rm B\kern-.05em{\sc i\kern-.025em b}\kern-.08em
    T\kern-.1667em\lower.7ex\hbox{E}\kern-.125emX}}

\newtheorem{proposition}{Proposition}

\newtheorem{lemma}{Lemma}
\newtheorem{example}{Example}

\title{The Double-Edged Sword of Behavioral Responses in Strategic Classification: Theory and User Studies}

\author{
 Raman Ebrahimi \\
  ECE, UC San Diego\\
  \texttt{raman@ucsd.edu} \\
   \And
 Kristen Vaccaro \\
  CSE, UC San Diego\\
  \texttt{kv@ucsd.edu} \\
  \And
 Parinaz Naghizadeh \\
  ECE, UC San Diego\\
  \texttt{parinaz@ucsd.edu} \\
}

\ifodd 1
\newcommand{\com}[1]{{\color{red}\textbf{Parinaz's Comment}: #1}}
\newcommand{\comr}[1]{{\color{orange}\textbf{Raman's Comment}: #1}}
\newcommand{\resp}[1]{{\color{cyan}\textbf{Response}: #1}} 
\else

\newcommand{\com}[1]{}
\newcommand{\comr}[1]{}
\newcommand{\resp}[1]{}
\fi

\begin{document}
\maketitle
\begin{abstract}
  When humans are subject to an algorithmic decision system, they can strategically adjust their behavior accordingly (``game'' the system). 
    While a growing line of literature on strategic classification has used game-theoretic modeling to understand and mitigate such gaming, these existing works consider standard models of \emph{fully rational} agents.   
    In this paper, we propose a strategic classification model that considers \emph{behavioral biases} in human responses to algorithms. We show how misperceptions of a classifier (specifically, of its feature weights) can lead to different types of discrepancies between biased and rational agents' responses, and identify when behavioral agents over- or under-invest in different features. We also show that strategic agents with behavioral biases can benefit or (perhaps, unexpectedly) harm the firm compared to fully rational strategic agents. We complement our analytical results with user studies, which support our hypothesis of behavioral biases in human responses to the algorithm. Together, our findings highlight the need to account for human (cognitive) biases when designing AI systems, and providing explanations of them, to strategic human in the loop. 
\end{abstract}

\section{Introduction}

As machine learning systems become more widely deployed, including in settings such as resume screening, hiring, lending, and recommendation systems, people have begun to respond to them strategically. Often, this takes the form of ``gaming the system'' or using an algorithmic system's rules and procedures to manipulate it and achieve desired outcomes. Examples include Uber drivers coordinating the times they log on and off the app to impact its surge pricing algorithm \citep{mohlmann2017hands}, and Twitter \citep{burrell2019users} and Facebook \citep{eslami2016first} users' decisions regarding how to interact with content given the platforms' curation algorithms. 

Game theoretical modeling and analysis have been used in recent years to formally analyze such strategic responses of humans to algorithms (e.g., \cite{Hardt2016strategic, Milli2019socialcost, Liu2020disparateequilibria}; see also Related Work). However, these existing works assume \emph{standard} models of decision making, where agents are fully rational when responding to algorithms; yet, humans exhibit different forms of cognitive biases in decision making \citep{kahnemann1979prospect}. Motivated by this, we explore the impacts \emph{behavioral biases} on agents' strategic responses to algorithms. 

We begin by proposing an extension of existing models of strategic classification to account for behavioral biases. Specifically, our model accounts for agents misperceiving (i.e., over-weighing or under-weighing) the importance of different features in determining the classifier's output. These may be known to agents in a full information game or can become available to them when the firm offers explanations through an Explainable AI (XAI) method which provides information about feature importance/contribution in the algorithm (e.g. SHAP~\citep{lundberg2017shap} or LIME~\citep{ribeiro2016lime}). We use this model to identify different forms of discrepancies that can arise between behavioral and fully rational agents' responses (Lemmas~\ref{lemma:band-optimization}-\ref{lemma:manhattan-cost-band}). We further identify conditions under which agents' behavioral biases lead them to over- or under-invest in specific features (Proposition~\ref{prop:under-invest-high-dim}). Moreover, we show that a firm's utility could increase or decrease when agents are behaviorally biased, compared to when they are fully rational (Proposition~\ref{prop:mismatch-actual-b}). While the former may be intuitively expected (behaviorally biased agents are less adept at gaming algorithms), the latter is more surprising; we further provide an intuitive explanation for this through a numerical example (Example~\ref{ex:firm-benefit-hurt}), highlighting the impact of agents' qualification states in determining the ultimate impact of agents' behavioral biases on the firm. 

Finally, by conducting a user study, we show that this type of behavioral bias is present when individuals interact with an AI decision assistant. Our study shows that individuals tend to underestimate the importance of the most crucial feature while overestimating the importance of the least important one. We also find that increasing the complexity of the model, either by adding more features or having unbalanced feature weights, amplifies this bias. Additionally, we observe other forms of cognitive biases (not captured by probability weighting biases), such as some individuals disproportionately investing in a feature with a lower starting point when feature weights are similar.

Together, our theoretical findings and user studies highlight the necessity of accounting for not just strategic responses but also cognitive biases when designing AI systems with human in the loop. 

\textbf{Summary of contributions.}

\textbullet\, We extend existing models of strategic classification to account for agents' cognitive biases in perceiving feature importance. 

\textbullet\, We analyze how these biases lead to over- or under-investment in certain features compared to fully rational agents. We further show that behaviorally biased agents can increase or decrease firm utility. 

\textbullet\, Through a user study, we confirm that cognitive biases influence human's understanding of and responses to AI systems, especially when they are given (explanations of) models with unbalanced feature weights and a higher number of features. 

\textbf{Related Work.} Our work is closely related to the literature on analyzing agents' responses to machine learning algorithms, when agents have full \citep{Hardt2016strategic, Perdomo2020performative, Milli2019socialcost, Hu2019disparate, Liu2020disparateequilibria, bechavod2022information, kleinberg202induce, alhanouti2024could, pmlr-v162-zhang22l, bechavod2021gaming} or partial information \citep{harris2022bayesian,cohen2024bayesian} about the algorithm, or principal's strategy \citep{haghtalab2023calibratedstackelberggameslearning}. While our base model of agents' strategic responses to (threshold) classifiers has similarities to those in some of these works (e.g., \cite{Hu2019disparate, Liu2020disparateequilibria}), we differ in our modeling of agent's \emph{behavioral} responses as opposed to fully \emph{rational} (non-behavioral) best responses considered in these works. 

The necessity of accounting for human biases in making AI assisted decisions \citep{rastogi2022deciding, nourani2021anchoring}, and various aspects of decision-making and model design \citep{Morewedge2023bias, zhu2024capturingcomplexityhumanstrategic, liu2024largelanguagemodelsassume,heidari2021perceptions, ethayarajh2024ktomodelalignmentprospect}has been considered in recent work. Among these, \cite{heidari2021perceptions} uses probability weighting functions to model human perceptions of allocation policies. We also consider (Prelec) weighting functions, but to highlight special cases of our results. We also differ from all these existing works in our focus on the \emph{strategic classification} problem.

Broadly, our research is also related to the area of explainable machine learning. While explanations can be helpful in increasing accountability, there is debate about the efficacy of existing explainability methods in providing correct and sufficient details in a way that helps users understand and act around these systems \citep{doshivelez2019accountabilityailawrole, kumar2020shapproblem, lakkaraju2020fool, adebayo2018sanity}. To complement these discussions, our work provides a formal model of how agents' behavioral biases may shape their responses to explanations (of feature importance) provided to them. We further confirm the presence of behavioral biases through a user study. Previous works have utilized user studies to assess interpretable models based on factors such as time spent, number of words, and accuracy \citep{lakkaraju2016decisionsets}, to establish the core principles of interpretability goals \citet{hong2020humanfactors}, and to assess the impact of model interpretability on predicting model outputs \citep{poursabzi2021manipulating}. In contrast, we assess how behavioral biases can result in human subjects' sub-optimal responses to interpretable models. We review additional related work in Appendix~\ref{sec:app-lit-review}.  

\section{Model and Preliminaries}\label{sec:model}
\textbf{Strategic Classification.} We consider an environment in which a \emph{firm} makes binary classification decisions on \emph{agents} with (observable) features $\mathbf{x}\in\mathbb{R}^n$ and (unobservable) true qualification states/labels $y\in\{0,1\}$, where label $y=1$ (resp. $y=0$) denotes qualified (resp. unqualified) agents. The firm uses a threshold classifier $h(\vx, (\vtheta, \theta_0))=\mathbf{1}(\vtheta^T\vx\geq \theta_0)$ to classify agents, where $\mathbf{1}(\cdot)$ denotes the indicator function, and $\vtheta=[\theta_1, \theta_2, \ldots, \theta_n]^T$ denotes the \emph{feature weights}; we assume features are normalized so that $\sum_i \theta_i=1$. 

Agents are strategic, in that they can respond to (``game'') this classifier. (As an example, in a college admission setting where grades are used to make admission decisions, students can study or cheat to improve their grades.) Formally, an agent with \emph{pre-strategic} features $\vx_0$ best-responds to classifier $(\vtheta, \theta_0)$ to arrive at the \emph{(non-behavioral) post-strategic} features $\vx_{\text{NB}}$ by solving the optimization problem:
\begin{align}\label{eq:agent-optimization}
    &\vx_\text{NB} := \argmax_\vx ~ rh(\vx, (\vtheta, \theta_0))-c(\vx, \vx_0) \notag\\
    &\text{subject to}\quad c(\vx, \vx_0)\le B
\end{align}
where $r>0$ is the reward of positive classification, $c(\vx, \vx_0)$ is the cost of changing feature vector $\vx_0$ to $\vx$, and $B$ is the agent's budget. We consider three different cost functions: \emph{norm-2 cost} (with $c(\vx, \vx_0) = \norm{\vx-\vx_0}_2^2=\sum_i (x_i-x_{i,0})^2$), \emph{quadratic cost} {(with $c(\vx, \vx_0) = \sum_i c_i(x_{i}-x_{0,i})^2$)}, and \emph{weighted Manhattan (taxicab) distance cost} (with $c(\vx, \vx_0)=\vc^T|\vx-\vx_0|=\sum_i c_i(|x_{i}-x_{0,i}|)$). Our analytical results are presented for the \emph{norm-2 cost}. We also characterize the agent's best-responses under other cost functions to highlight that similar agent behavior can be seen under them. 

Anticipating the agents' responses, the firm can choose the optimal (non-behavioral) classifier threshold by solving 
$(\vtheta_\text{NB}, \theta_{0, \text{NB}}) := \argmin_{(\vtheta, \theta_0)} \E_{\vx\sim\mathcal{D}(\vtheta, \theta_0)}[l(\vx, (\vtheta, \theta_0))]$, where $\mathcal{D}(\vtheta, \theta_0)$ is the post-strategic feature distribution of agents responding to classifier $(\vtheta, \theta_0)$, and $l(\cdot)$ is the firm's loss function (e.g., weighted sum of TP and FP costs). 

\textbf{Behavioral Responses.} We extend the strategic classification model to allow for behavioral responses by agents. Formally, recall that we normalize the feature weight vector $\vtheta=[\theta_1, \theta_2, \ldots, \theta_n]^T$ to have $\sum_i \theta_i=1$. We interpret it as a probability vector, and assume that behaviorally biased agents misperceive $\vtheta$ as $\vw(\vtheta)$, where $\vw(\cdot)$ is a function capturing their biases. One choice for $\vw(\cdot)$ can be $\evw_j(\vtheta) = p(\sum_{i=1}^j \theta_i)-p(\sum_{i=1}^{j-1} \theta_i)$~\citep{gonzalez1999shape} where $p(z)=\exp(-(-\ln(z))^\gamma)$ is the widely used probability weighting function introduced by \citet{Prelec1998} with $\gamma$ reflecting the intensity of biases.

Now, a behaviorally biased agent with {pre-strategic} features $\vx_0$ best-responds to classifier $(\vtheta, \theta_0)$ to arrive at the \emph{behavioral post-strategic} features $\vx_{\text{B}}$ by solving:
\begin{align}\label{eq:agent-optimization-behavioral}
    &\vx_\text{B} := \argmax_\vx ~ rh(\vx, (\vw(\vtheta), \theta_0))-c(\vx, \vx_0)\notag\\
    &\text{subject to} \quad c(\vx, \vx_0)\le B
\end{align}
Note that the agent now responds to a \emph{perceived feature weights} $(\vw(\vtheta), \theta_0)$. In return, while always accounting for agents' strategic behavior (``gaming''), we assume the firm may or may not be aware that agents have behavioral biases when gaming the system. Specifically, let $\sL(\vtheta', (\vtheta, \theta_{0})):= \E_{\vx\sim\mathcal{D}(\vtheta', \theta_{0})}[l(\vx, (\vtheta, \theta_{0}))]$ denote a firm's expected loss when it implements a classifier $(\vtheta, \theta_{0})$ and agents respond to a (potentially different) classifier $(\vtheta', \theta_{0})$. Then, if a firm is aware of strategic agents' behavioral biases, it selects the threshold $(\vtheta_\text{B}, \theta_{0, \text{B}}) := \arg\min_{(\vtheta, \theta_0)} \sL(\vw(\vtheta), (\vtheta, \theta_0))$ and incurs a loss $\sL(\vw(\vtheta_{B}), (\vtheta_{B}, \theta_{0, \text{B}}))$. On the other hand, a firm that assumes agents are fully rational selects the threshold classifier $(\vtheta_\text{NB}, \theta_{0, \text{NB}})$, yet incurs the loss $\sL(\vw(\vtheta_\text{NB}), (\vtheta_\text{NB}, \theta_{0, \text{NB}}))$. 

\section{Agents' Strategic Responses}\label{sec:agetns-response}

We first fix the classifier $(\vtheta, \theta_0)$, and compare fully rational (non-behavioral) and behavioral agents' strategic responses to it. The following Lemma characterizes $\vx_\text{NB}$ (the solution to \eqref{eq:agent-optimization}) and $\vx_\text{B}$ (the solution to \eqref{eq:agent-optimization-behavioral}) under the norm-2 cost.
\begin{lemma}\label{lemma:band-optimization}
    Let $d(\vx_0, \vtheta, \theta_0)=\frac{\theta_0-\vtheta^T\vx_0}{\norm{\vtheta}_2}$  denote $\vx_0$'s distance to the hyperplane $\vtheta^T\vx=\theta_0$. Then, for an agent with starting feature vector $\vx_0$, if $0 < d(\vx_0, \vtheta, \theta_{0}) \le B$, 
    \begin{align*}
        \vx_\text{NB} = \vx_0 + d(\vx_0, \vtheta, \theta_{0})\vtheta~.
    \end{align*}
    Otherwise, $\vx_\text{NB} = \vx_0$. For behaviorally biased agents, $\vx_{B}$ is obtained similarly by replacing $\vtheta$ with $\vw(\vtheta)$.
\end{lemma}

Figure~\ref{fig:BR-illustration} illustrates the strategic agents' best-responses of Lemma~\ref{lemma:band-optimization}, in a two-dimensional feature space, when they are non-behavioral (Fig.~\ref{fig:NB-arrows}) and when they are behavioral (Fig.~\ref{fig:B-arrows}). We first note that the subset of agents with non-trivial responses to the classifier, as identified in  Lemma~\ref{lemma:band-optimization}, are in a band below the decision boundary. Given the overlaps of these bands under non-behavioral and behavioral responses, there are 6 regions of interest where biased agents' best-responses defer from rational agents (Fig.~\ref{fig:highlighted}). In regions \framebox(7,9){1} and \framebox(7,9){6}, agents invest no effort in manipulating their features when they are behaviorally biased, whereas they do when fully rational; the reasons differ: agents in \framebox(7,9){1} believe they are accepted without effort, while those in \framebox(7,9){6} believe they do not have sufficient budget to succeed. Agents in regions \framebox(7,9){2} and \framebox(7,9){5} manipulate their features unnecessarily (they would not, had they been fully rational), and again, for different reasons: agents in \framebox(7,9){2} are not accepted even at their highest effort level, while those in \framebox(7,9){5} believe they must reach the boundary but they would be accepted regardless of their effort. Finally, in region \framebox(7,9){3}, agents \emph{undershoot} the actual boundary (i.e., exert less effort than needed due to their biases), while those in region \framebox(7,9){4} \emph{overshoot} (i.e., exert more effort than needed to get accepted). 

\begin{figure*}[ht]
    \centering
    \begin{subfigure}[t]{0.3\textwidth}
        \includegraphics[width=0.82\textwidth]{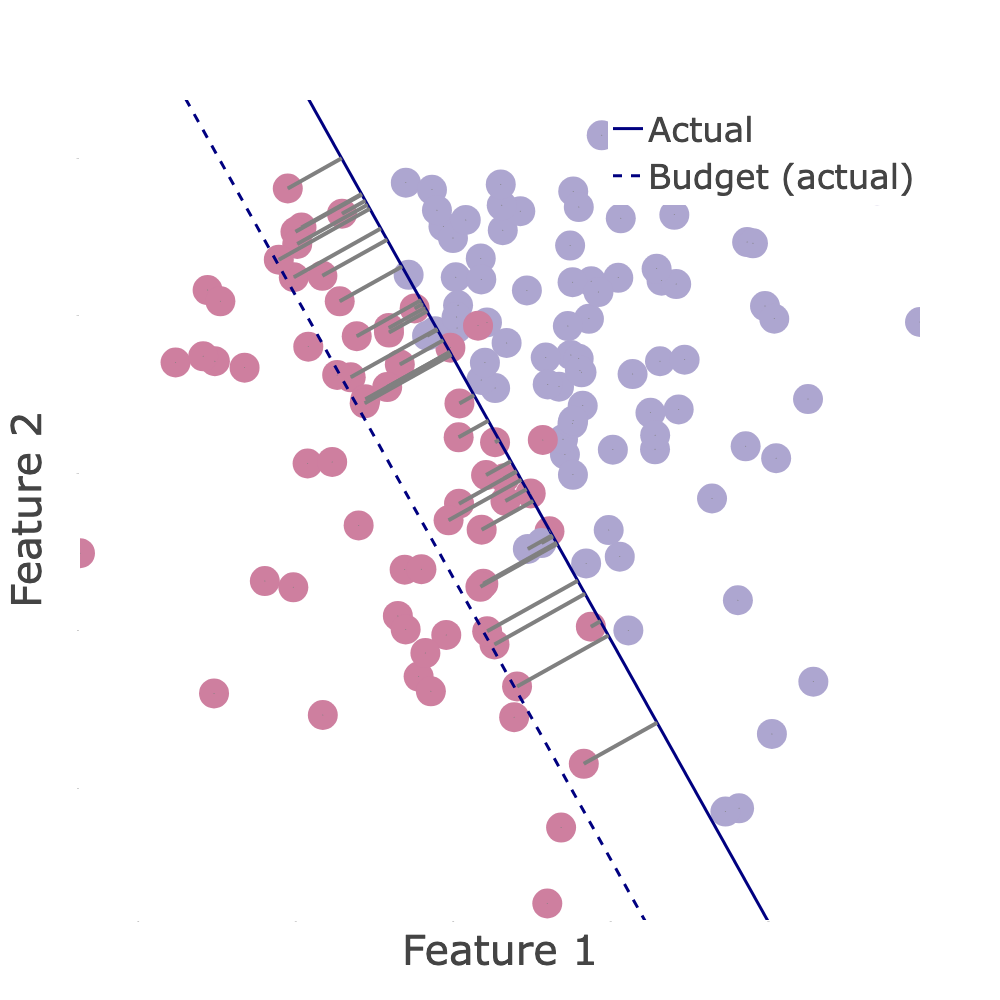}
            \caption{Rational strategic response}
        \label{fig:NB-arrows}
    \end{subfigure}
    \begin{subfigure}[t]{0.3\textwidth}
        \includegraphics[width=0.82\textwidth]{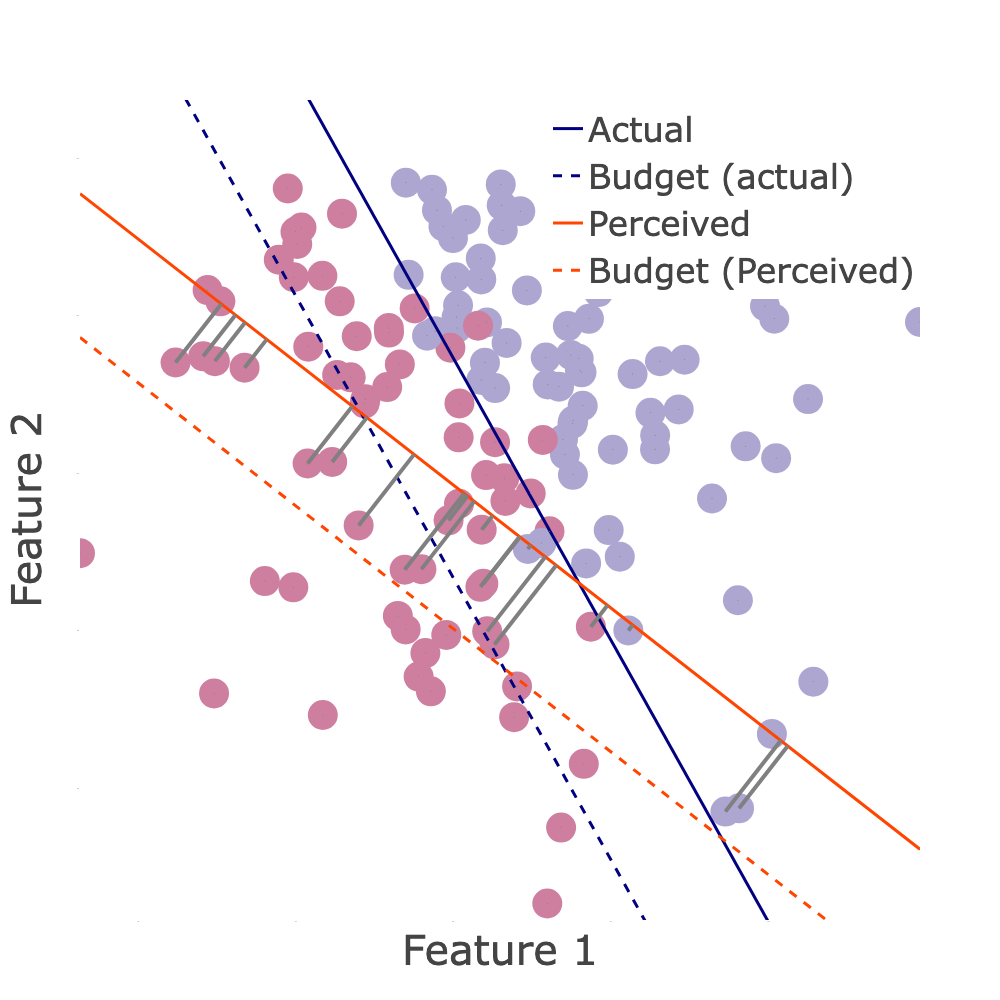}
        \caption{Biased strategic response}
        \label{fig:B-arrows}
    \end{subfigure}
    \begin{subfigure}[t]{0.3\textwidth}
        \includegraphics[width=0.82\textwidth]{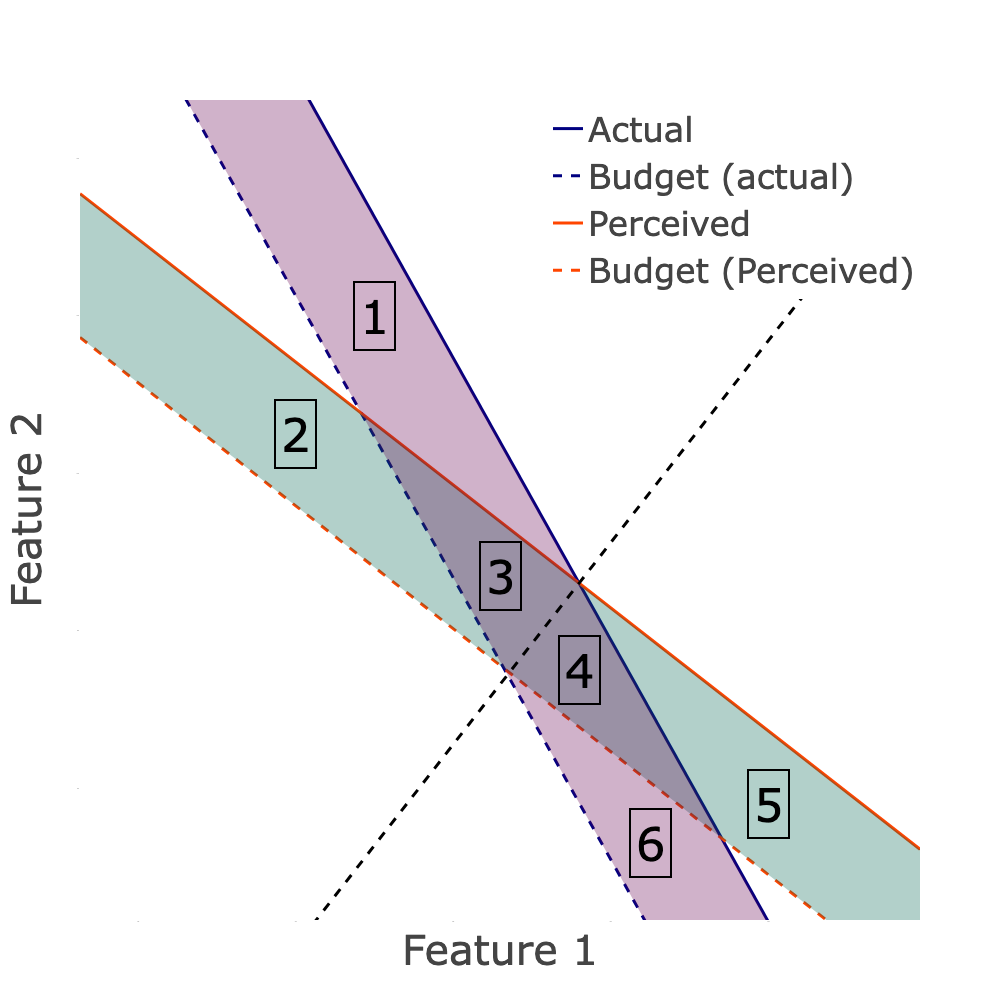}
            \caption{Differing strategic responses}
        \label{fig:highlighted}
    \end{subfigure}
    \caption{(a) Fully rational and (b) Biased responses, and (c) Classes of differing actions under quadratic costs.}
    \label{fig:BR-illustration}
\end{figure*} 

In the following proposition, we further investigate best-responses in region \framebox(7,9){4} (resp. region \framebox(7,9){3}) and identify which features behavioral agents over-invest in (resp. under-invest in) that leads to them overshooting (resp. undershooting) past the true classifier $(\vtheta, \theta_0)$. 
\begin{proposition}\label{prop:under-invest-high-dim}
Consider an agent with features $\vx_0$, facing classifier $(\vtheta, \theta_0)$, and with a misperceived $\vw(\vtheta)$. Let $\theta_{\max}=\max_i \theta_i$, $d(\vx_0, \vtheta, \theta_0)=\frac{\theta_0-\vtheta^T\vx_0}{\norm{\vtheta}_2}$, and  $\delta^{\text{NB}}_i=x_{\text{NB},i}-x_{0,i}$ and $\delta^{\text{B}}_i=x_{\text{B},i}-x_{0,i}$ denote the changes in feature $i$ after best-responses. Then:\\[2pt]
    \hspace*{0.1in}{\textbf{(1)}} If $d(\vx_0, \vw(\vtheta), \theta_0)\le d(\vx_0, \vtheta, \theta_0)$ and $w(\theta_i)<\theta_i$, then $\delta_i^{\text{B}}<\delta_i^{\text{NB}}$.\\[2pt]
    \hspace*{0.1in}\textbf{(2)} If $d(\vx_0, \vtheta, \theta_0)\le d(\vx_0, \vw(\vtheta), \theta_0)$ and $\theta_i<w(\theta_i)$ then $\delta_i^{\text{NB}}<\delta_i^{\text{B}}$.\\[2pt]
    \hspace*{0.1in}\textbf{(3)} For the special case of a Prelec function, we further have: If $d(\vx_0, \vtheta, \theta_0) \le e^{\gamma^\frac{1}{1-\gamma}-\gamma^\frac{\gamma}{1-\gamma}} d(\vx_0, \vw(\vtheta), \theta_0)$ and $w(\theta_{\max})<\theta_{\max}$, then
    $\delta_{\max}^{\text{NB}}<\delta_{\max}^{\text{B}}$. 
\end{proposition}
Intuitively, the proposition states that agents who perceive the decision boundary to be closer to them than it truly is (regions \framebox(7, 9){2} and \framebox(7, 9){3} in Figure~\ref{fig:highlighted}) will under-invest in the features for which they underestimate the importance. Similarly, agents that perceive the boundary to be farther (regions \framebox(7, 9){4} and \framebox(7, 9){5} in Figure~\ref{fig:highlighted}) will over-invest in the features for which they overestimate the importance. 

\subsection{Alternative Cost Functions}
We next show that regions of differing responses between behavioral and non-behavioral agents, similar to those depicted in Figure~\ref{fig:highlighted}, will also emerge under quadratic and weighted Manhattan cost functions. 

\paragraph{Quadratic Cost Function.}
The following lemma characterizes the post-strategic features under the quadratic cost $c(\vx, \vx_0)=\sum_i c_i(x_{i}-x_{i,0})^2$.
\begin{lemma}\label{lemma:quad-cost-band}
    Let $\mC$ denote a diagonal matrix with $c_i$'s as its diagonal, and let $\vy$ denote the feature vectors satisfying $\vtheta^T\vy = \theta_0$. For an agent with starting feature vector $\vx_0$, if $\vx_0$ is in the n-dimensional ellipsoid described by $B$ and $\mC$, i.e., if $(\vy-\vx_0)^T\mC (\vy-\vx_0)\le B$,
    \begin{align*}
        x_{\text{NB}, i} = \frac{\theta_{0}-\vtheta^T\vx_0}{\sum_j \frac{\theta_j^2}{c_j}}\cdot \frac{\theta_i}{c_i}+x_{0,i}~.
    \end{align*}
    Otherwise, $\vx_\text{NB}=\vx_0$. For behaviorally biased agents, $\vx_\text{B}$ is obtained similarly by replacing $\vtheta$ with $\vw(\vtheta)$.
\end{lemma}

Figure~\ref{fig:BR-illustration-quad-cost} illustrates the best-responses of Lemma~\ref{lemma:quad-cost-band} for rational (non-behavioral) and biased (behavioral) agents. Specifically, the condition in Lemma~\ref{lemma:quad-cost-band} constructs an $n$-dimensional ellipsoid around every point on the line $\vtheta^T\vx = \theta_{0}$, containing agents who have sufficient budget to strategically change their features to reach that particular point on the boundary, with the coefficients $c_i$ determining the scaling along each axis. 
Since the scaling of the ellipsoid does not depend on the point of the line we are focusing on, the union of these ellipsoids (determining the set of all agents who can afford to be classified positively through gaming) forms a band below the line $\vtheta^T\vx=\theta_0$. Note that this band differs from the one in Lemma~\ref{lemma:band-optimization}. 

\begin{figure}[ht]
    \centering
    \begin{subfigure}[t]{0.4\textwidth}
        \includegraphics[width=0.9\textwidth]{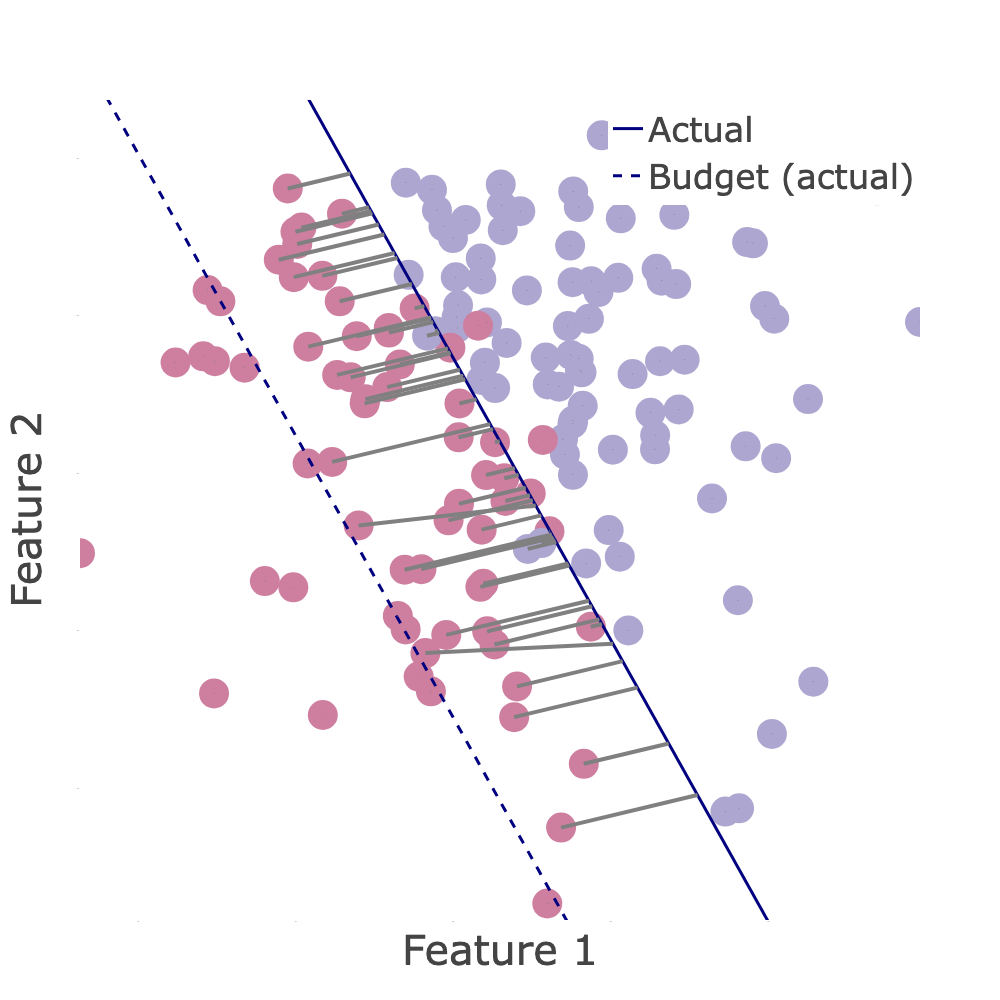}
            \caption{Rational response}
        \label{fig:NB-arrows-quad-cost}
    \end{subfigure}
    \begin{subfigure}[t]{0.4\textwidth}
        \includegraphics[width=0.9\textwidth]{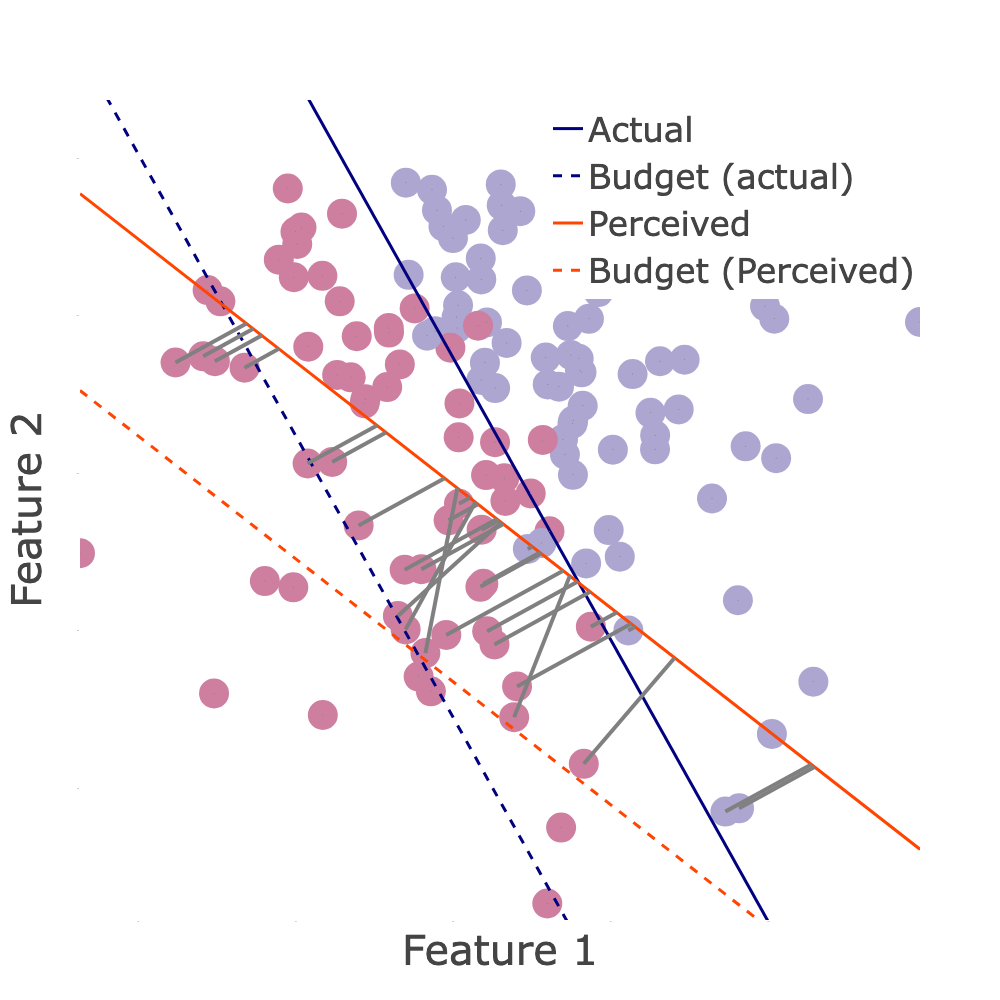}
        \caption{Biased response}
        \label{fig:B-arrows-quad-cost}
    \end{subfigure}
    \caption{Strategic responses under quadratic costs.}
    \label{fig:BR-illustration-quad-cost}
\end{figure} 

\paragraph{Weighted Manhattan Distance Cost Function.}
The following lemma characterizes the post-strategic features under the weighted Manhattan cost $c(\vx, \vx_0)=\sum_i c_i|x_{i}-x_{i,0}|$.
\begin{lemma}\label{lemma:manhattan-cost-band}
    Let $\ve_i$ be the unit vector with 1 in the $i$\textsuperscript{th} coordinate and 0 elsewhere, and $k = \arg\min_i \frac{c_i}{\theta_i}$. For an agent with starting feature $\vx_0$, if $\vtheta^T\vx_0 + \frac{\theta_k}{c_k}B \ge \theta_0$,
    \begin{align*}
        \vx_\text{NB} = \vx_0 + (\theta_0-\vtheta^T\vx_0)\frac{c_k}{\theta_k}\ve_k~.
    \end{align*}
    Otherwise, $\vx_\text{NB}=\vx_0$. For behaviorally biased agents, $\vx_\text{B}$ is obtained similarly by replacing $\vtheta$ with $\vw(\vtheta)$.
\end{lemma}

Figure~\ref{fig:BR-illustration-lin-cost} illustrates the best-responses of Lemma~\ref{lemma:manhattan-cost-band} for rational (non-behavioral) and biased (behavioral) agents.
Again, the set of agents who can afford to game the system to receive a positive classification, as identified in Lemma~\ref{lemma:manhattan-cost-band}, is a band below the classifier (but different from those of Lemmas~\ref{lemma:band-optimization} and \ref{lemma:quad-cost-band}). In particular, under this cost, agents spend all their budget on changing the feature with the most ``bang-for-the-buck'' $\frac{c_i}{\theta_i}$ (or perceived bang-for-the-buck $\frac{c_i}{w_i(\vtheta)}$). As seen in the two-dimensional illustration in Figure~\ref{fig:BR-illustration-lin-cost}, this means that while it is optimal for rational agents to invest only in feature 2, those with behavioral bias believe feature 1 has a better return, leading to a sub-optimal response by them. We also note that even though the movements of agents in the specified band are different from the movement for the norm-2 cost, the bands form the same regions of differing responses as in Figure~\ref{fig:BR-illustration}, where agents overshoot, undershoot, do nothing at all, or needlessly change their features, when they are behaviorally biased. 

\begin{figure}[ht]
    \centering
    \begin{subfigure}[t]{0.4\textwidth}
        \includegraphics[width=0.9\textwidth]{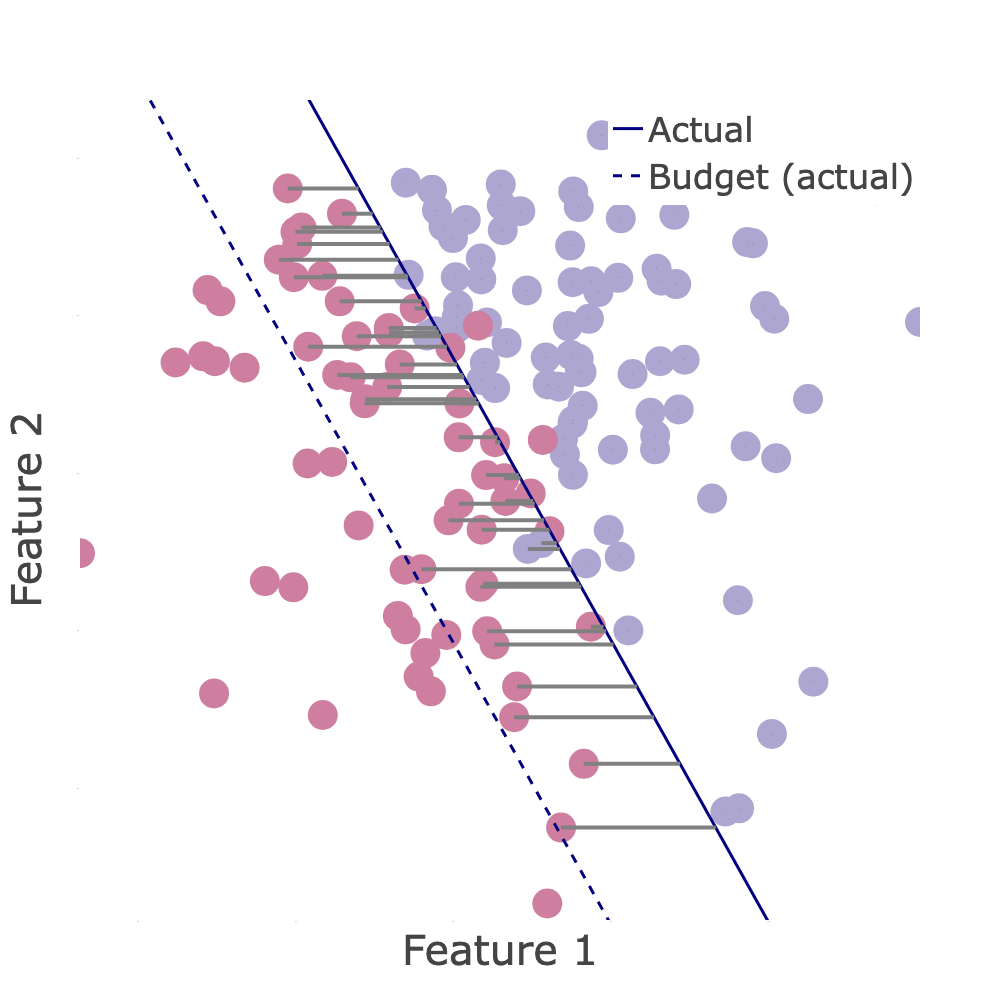}
            \caption{Rational response}
        \label{fig:NB-arrows-lin-cost}
    \end{subfigure}
    \begin{subfigure}[t]{0.4\textwidth}
        \includegraphics[width=0.9\textwidth]{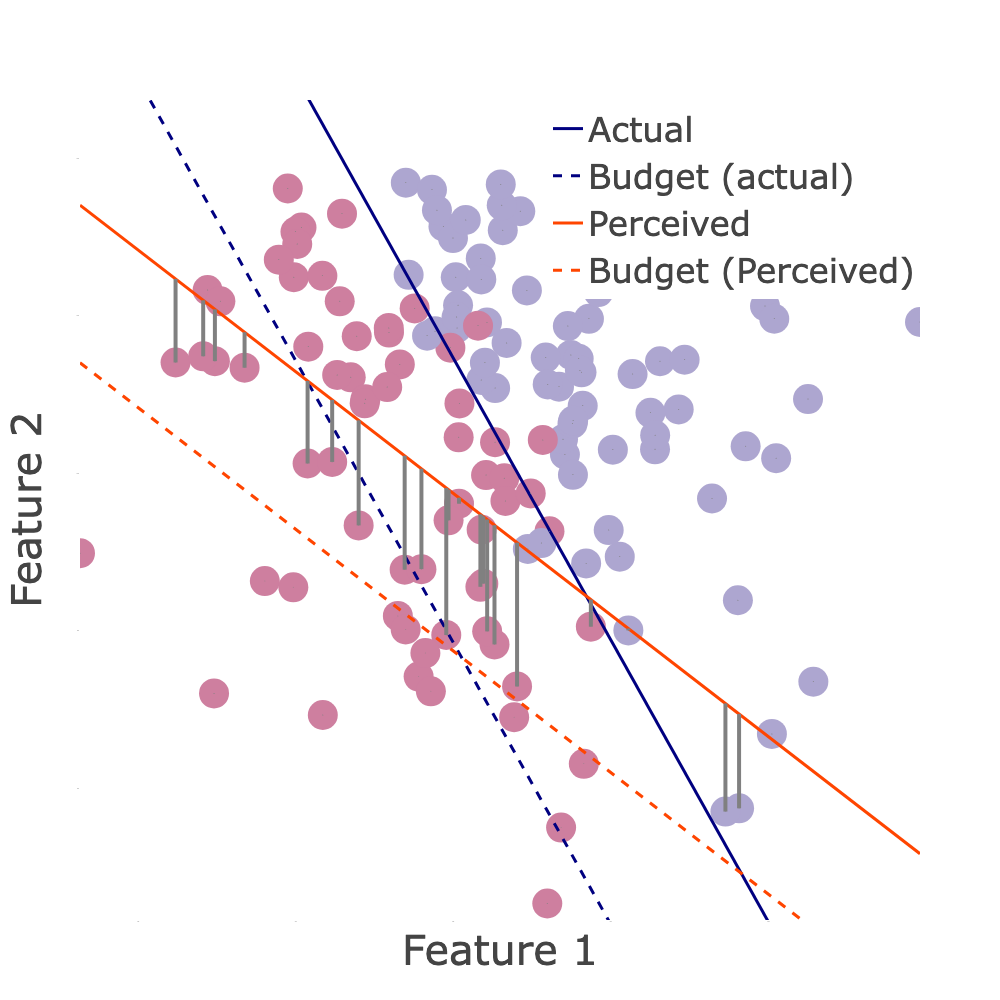}
        \caption{Biased response}
        \label{fig:B-arrows-lin-cost}
    \end{subfigure}
    \caption{Strategic responses under Manhattan costs.}
    \label{fig:BR-illustration-lin-cost}
\end{figure}  

\paragraph{Cost Function in the User Studies.} For our human subject experiments in Section~\ref{sec:user-study}, we describe the cost of changing features to participants through a \emph{piecewise linear cost function}. This can be viewed as an approximation of a quadratic cost using a step function with a weighted Manhattan cost at each step, with the approximation improving as the number of steps increases (see the two-dimensional illustration in Figure~\ref{fig:2d-approx-illustration}). Specifically, in our user experiments, we break the budget $B$ into three steps of increments $B_1$, $B_2$, and $B_3$ with $B_1+B_2+B_3=B$, and assign a constant cost $c_1$, $c_2$, and $c_3$ for changing features at each increment. This means that in each step, agents face a weighted Manhattan cost, but overall, the cost is not fixed, and investing in a single feature is not optimal. 

\begin{figure}[ht]
    \centering
    \includegraphics[width=0.5\linewidth]{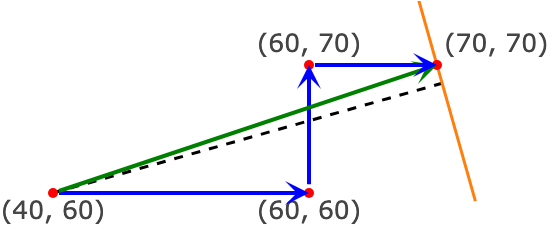}
    \caption{Strategic responses under a quadratic cost (green) vs. a piece-wise linear cost function (blue).}
    \label{fig:2d-approx-illustration}
\end{figure}

\section{Firm's Response}\label{sec:firm-response}
We next consider the firm's optimal choice of a classifier, given agents' strategic responses, and its impact on the firm's utility and agents' welfare. Intuitively, one might expect a firm to ultimately benefit from agents' behavioral responses (in contrast to fully rational responses) as behavioral agents are less adept at gaming the algorithm. However, in this section, we show that this is not always true. Intuitively, as demonstrated in Section~\ref{sec:agetns-response}, behavioral agents may overshoot or undershoot the threshold when gaming the algorithm (compared to rational agents); this includes both qualified (label 1) and unqualified (label 0) agents. We show that there exist scenarios in which a relatively higher number of behaviorally biased qualified agents end up below the threshold (due to not trying or undershooting) while relatively more unqualified agents overshoot and end up accepted by the classifier; the combination of these factors can decrease the firm's utility. In other words, perhaps unexpectedly, in these situations, the firm would prefer rational agents, who are better at gaming the system, to behaviorally biased agents, who are worse at gaming the system. 
The following example numerically illustrates this. 

\begin{example}\label{ex:firm-benefit-hurt}
Consider a setting where we have a 2D feature space and qualified (resp. unqualified) agents are sampled from a normal distribution $\mathcal{N}(\vmu_{1}, \Sigma_1)$ (resp. $\mathcal{N}(\vmu_0, \Sigma_0)$). We consider three scenarios; the first two scenarios only differ in the mean $\vmu_{1}$ choice, and the third scenario differs with these in $\vmu_{1}$, $\Sigma_1$, $\Sigma_0$, and $B$ (see Appendix~\ref{sec:app-numerical-details} for details). The first two scenarios (top and middle rows in Figure~\ref{fig:firm-benefit-hurt-dist}) are baselines: we consider an \emph{oblivious} firm that chooses its classifier without accounting for any strategic response (whether rational or behavioral) from agents. This helps us hone in on the impacts of agents' qualification states on the firm's utility. Then, in the third scenario (bottom row in Figure~\ref{fig:firm-benefit-hurt-dist}), we consider a firm that is aware of strategic behavior (and any behavioral biases) by agents and optimally adjusts its classifier. For each scenario, Figure~\ref{fig:firm-benefit-hurt-dist} illustrates the distribution of agents' features for pre-strategic (left panel), post-strategic non-behavioral responses (middle panel), and post-strategic behaviorally-biased responses (right panel). The firm's utility in each case is shown at the top of the corresponding subplot.

We start with the baselines (an oblivious firm that keeps the classifier fixed). In the top row scenario, the firm is negatively impacted by agents' behavioral biases, while in the middle row scenario, the firm benefits from agents' biases (both compared to the fully rational setting). The reason for this difference is that there are more qualified agents than unqualified ones who reach the threshold in non-biased responses. On the other hand, under biased responses, there are more unqualified agents who pass the threshold, regardless of their bias (those in region \framebox(7, 9){3} in Fig.~\ref{fig:highlighted}) in the top row scenario. Behavioral responses by these agents negatively impact the firm, as it leads to these qualified agents no longer being accepted. 

Next, we consider the (non-oblivious) firm that adjusts its classifier optimally (accounting for strategic responses and behavioral biases, if any). We observe that even though the firm is aware of agents' bias, its loss is higher than the case of rational responses. As seen in the left panel of the bottom row of Figure~\ref{fig:firm-benefit-hurt-dist}, more regions can impact the loss than in Figure~\ref{fig:BR-illustration}. The most important regions in this scenario are the areas accepted by $\vtheta_\text{NB}$ but not by $\vtheta_\text{B}$ (after response), and vice versa. As there are more qualified than unqualified agents in these two regions, the firm is negatively impacted by agents' bias (compared to fully rational agents) even though the firm is aware of the bias.

    \begin{figure}[ht]
        \centering
        \includegraphics[width=0.95\linewidth]{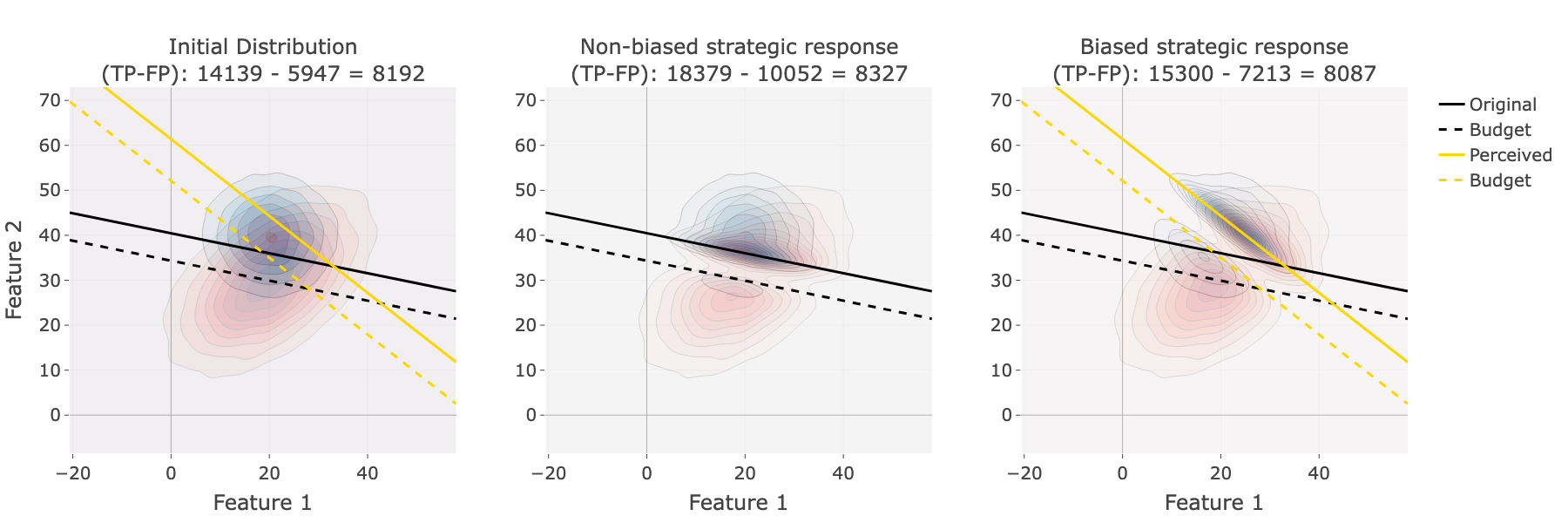}
        \includegraphics[width=0.95\linewidth]{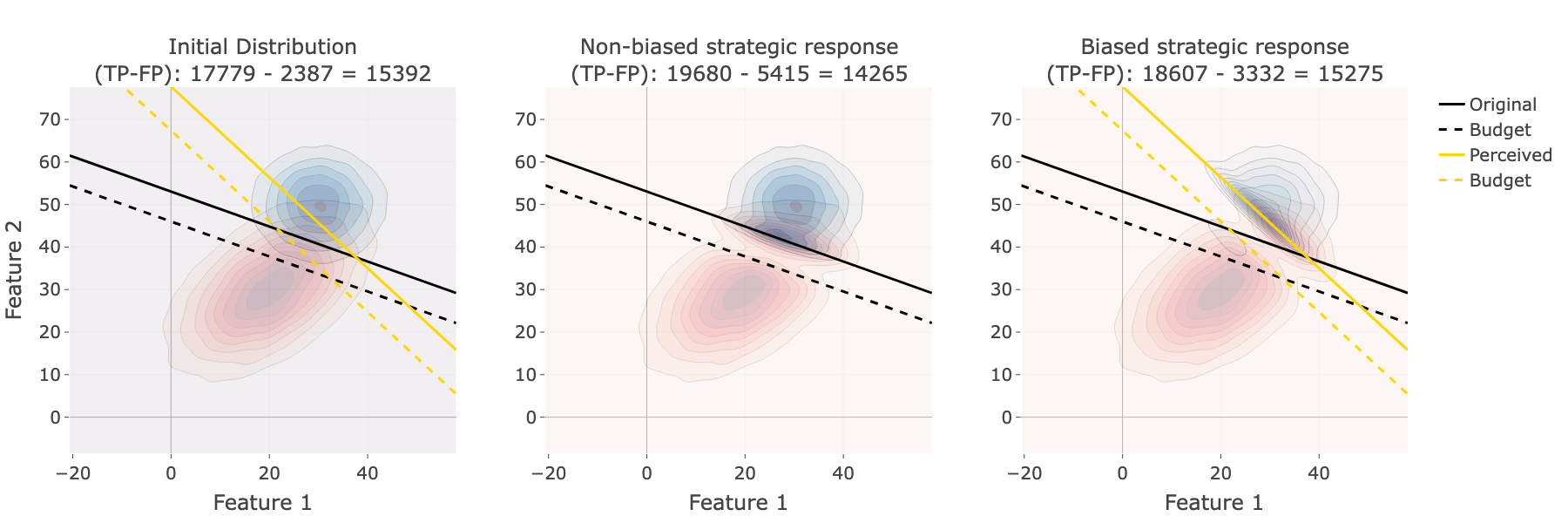}
        \includegraphics[width=0.95\linewidth]{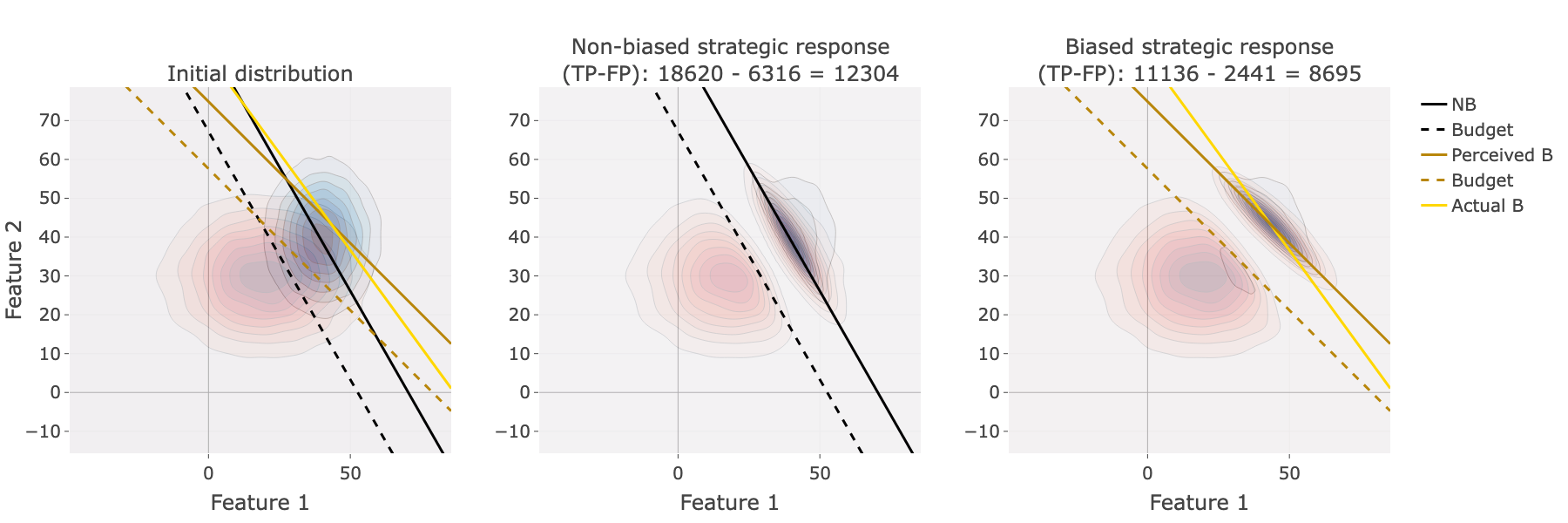}
        \caption{An oblivious firm may have lower (top) or higher (middle) utility when agents are biased (vs. rational). A non-oblivious firm may still have a lower utility when agents are biased (bottom).}
        \label{fig:firm-benefit-hurt-dist}
    \end{figure}
\end{example}

The next proposition formalizes the above intuition. 

\begin{proposition}\label{prop:mismatch-actual-b} 
    Consider a loss function $l(\vx, (\vtheta, \theta_0))=-u^+\text{TP}+u^-\text{FP}$. Let the pdf of label $y$ agents' feature distribution be $f_y(\vx)$, and the number of label $y$ agents be $\alpha_0$. Let $\mathcal{H}(\vtheta, \theta_0)$ denote the set of agents that satisfy $(1-\sigma(\vtheta))\theta_0\le(\vtheta-\sigma(\vtheta)\vw(\vtheta))^T\vx$, where $\sigma(\vtheta) \coloneqq \frac{\vtheta^T\vw(\vtheta)}{\norm{\vw(\vtheta)}^2}$\footnote{Note that $\sigma(\vtheta)=\frac{\norm{\vtheta}_2}{\norm{\vw(\vtheta)}_2}\cos(\alpha)$ where $\alpha$ is the angle between the actual and perceived decision boundaries. The larger $\alpha$ is, the lower $\sigma(\vtheta)$ is, indicating a more intense bias.}, and the set of agents that attempt to game the algorithm as $\sA(\vtheta, \theta_0) = \{\vx_0: \theta_0 - B \le \vtheta^T\vx_0 < \theta_0 \}$. Denote the set of accepted (resp. rejected) agents by $(\vtheta, \theta_0)$ with $\sY(\vtheta, \theta_0)$ (resp. $\sN(\vtheta, \theta_0)$). Define the sets $\sS(\vtheta_\text{NB}, \theta_{0, \text{NB}}) \coloneqq \sA(\vtheta_\text{NB}, \theta_{0, \text{NB}})/(\sA(\vtheta_\text{NB}, \theta_{0, \text{NB}})\cap\mathcal{H}(\vtheta_\text{NB}, \theta_{0, \text{NB}}))$, $\sT_1=(\sY(\vtheta_\text{NB}, \theta_{0,\text{NB}})\cup \sA(\vtheta_\text{NB}, \theta_{0,\text{NB}}))\cap \sN(\vtheta_\text{B}, \theta_{0,\text{B}})$, and $\sT_2 = (\mathcal{H}(\vtheta_\text{B}, \theta_{0,\text{B}})\cap \sA(\vw(\vtheta_\text{B}), \theta_{0,\text{B}}))\cup ( (\sY(\vtheta_\text{B}, \theta_{0,\text{B}}) \cap \sN(\vtheta_\text{NB}, \theta_{0,\text{NB}}))/\sA(\vtheta_\text{NB}, \theta_{0,\text{NB}}) )$. Then:\\[2pt]
    \hspace*{0.1in}{\textbf{(a)}} If $\int_{x\in \sS(\vtheta_\text{NB}, \theta_{0, \text{NB}})} u^- f_0(\vx)\alpha_0 d\vx \le \int_{x\in \sS(\vtheta_\text{NB}, \theta_{0, \text{NB}})} u^+ f_1(\vx)\alpha_1 d\vx $ we can say: 
    \begin{align}\label{eq:firm-loss-comp-benefit}
        &\sL(\vw(\vtheta_\text{B}), (\vtheta_\text{B}, \theta_{0, \text{B}}))\le \sL(\vw(\vtheta_\text{NB}), (\vtheta_\text{NB}, \theta_{0, \text{NB}}))\le \sL(\vtheta_\text{NB}, (\vtheta_\text{NB}, \theta_{0, \text{NB}}))
    \end{align}
    \hspace*{0.1in}{\textbf{(b)}} If $\int_{x\in \sS(\vtheta_\text{NB}, \theta_{0, \text{NB}})} u^+ f_1(\vx)\alpha_1 d\vx \le \int_{x\in \sS(\vtheta_\text{NB}, \theta_{0, \text{NB}})} u^- f_0(\vx)\alpha_0 d\vx $ we can say: 
    \begin{align}\label{eq:firm-loss-comp-hurt}
        &\max\{\sL(\vtheta_\text{NB}, (\vtheta_\text{NB}, \theta_{0, \text{NB}})), \sL(\vw(\vtheta_\text{B}), (\vtheta_\text{B}, \theta_{0, \text{B}}))\}\le \sL(\vw(\vtheta_\text{NB}), (\vtheta_\text{NB}, \theta_{0, \text{NB}}))
    \end{align}
    \hspace*{0.1in}{\textbf{(c)}} If $\int_{\vx\in\sT_1}(-u^+f_1(\vx)\alpha_1+u^-f_0(\vx)\alpha_0)d\vx \le \int_{\vx\in\sT_2}(-u^+f_1(\vx)\alpha_1+u^-f_0(\vx)\alpha_0)d\vx$ we can say:
    \begin{align}\label{eq:firm-loss-comp-hurt-NB-B}
        &\sL(\vtheta_\text{NB}, (\vtheta_\text{NB}, \theta_{0, \text{NB}}))\le \sL(\vw(\vtheta_\text{B}), (\vtheta_\text{B}, \theta_{0, \text{B}}))\le \sL(\vw(\vtheta_\text{NB}), (\vtheta_\text{NB}, \theta_{0, \text{NB}}))
    \end{align}
\end{proposition}
Part (a) states that if a firm is unaware of agents' behavioral biases, it will suffer a lower loss when the population is biased compared to fully rational. This is the intuitively expected scenario (behaviorally biased agents are less adept than fully rational ones at gaming the algorithm). On the other hand, statement (b) reflects the less expected outcome: a firm unaware of behavioral biases will have \emph{lower} utility when agents are biased compared to if they had been fully rational (as more \emph{qualified} than \emph{unqualified} agents undershoot the threshold under this case's condition). Statement (c) further compares the unaware firm with an aware firm and provides a condition where an aware firm's minimal loss is higher than the non-biased minimal loss. This condition relies on the \emph{difference} of qualified and unqualified agents in two regions.

\textbf{Agents' Welfare:} We end this section by comparing the impacts of behavioral biases on agents' welfare (sum of their utilities). As a baseline, note that if the firm was oblivious to agents' strategic responses and did not adjust the classifier, agents would have lower welfare when they are behaviorally biased (compared to when rational). This is an expected outcome since behaviorally biased agents are worse at gaming the algorithm and respond sub-optimally. But, perhaps more unexpectedly, when the firm adjusts its classifier in response to agents' strategic behavior and behavioral biases, various scenarios can occur. For instance, in the bottom row of Figure~\ref{fig:firm-benefit-hurt-dist}, qualified agents have \emph{higher} social welfare when they are behaviorally biased compared to if they had been rational. We provide additional details on reasons for this in Appendix~\ref{app:welfare}.

\section{User Study}\label{sec:user-study}
\vspace{-0.1in}
\subsection{Study Design, Participants, and Setup}\label{sec:study-design}
To understand how human biases affect their responses to algorithms, we conducted a large-scale online survey where participants completed a task with an optimal solution. We then measured how their answers differed from the optimal solution to assess bias. Similar to previous sections, we focus on how participants perceive the feature weights and how this influences their response to the algorithm. Surveys generally took {5.5} minutes, and participants were paid {\$16 per hour}. The study design was reviewed by our IRB, and the full protocol is included in the supplementary materials.

\noindent\textbf{Experimental Design.} We used a $2\times2$ between-subjects design, varying the number of features and the feature weights. Participants were randomly assigned to a condition and shown a single explanation (shown in Figure~\ref{fig:feature-based-explanations}), with \textit{a}) either two or four features and \textit{b}) either balanced or unbalanced feature weights. 

\noindent\textbf{Procedure.} All participants were shown a truthful and complete explanation from an interpretable ML algorithm. Each participant was asked to complete a task based on the information provided in the explanation. After completing the task, participants were asked questions about their understanding, trust, satisfaction, and performance, common measures in explainability evaluations~\citep{mohseni2021survey}. 

\noindent\textbf{Explanation.} Based on prior work, which has found that feature importance is the most used explanation \citep{systematic2023nauta} and that visualization can help users understand explanations \citep{peeking2018adadi}, we developed an explanation that shows feature weights in a bar graph (Figure~\ref{fig:feature-based-explanations}). The system uses the displayed feature weights as an interpretable ML algorithm. 

\begin{figure}[ht]
    \centering 
    \hspace{.5em}\includegraphics[width=0.48\textwidth]{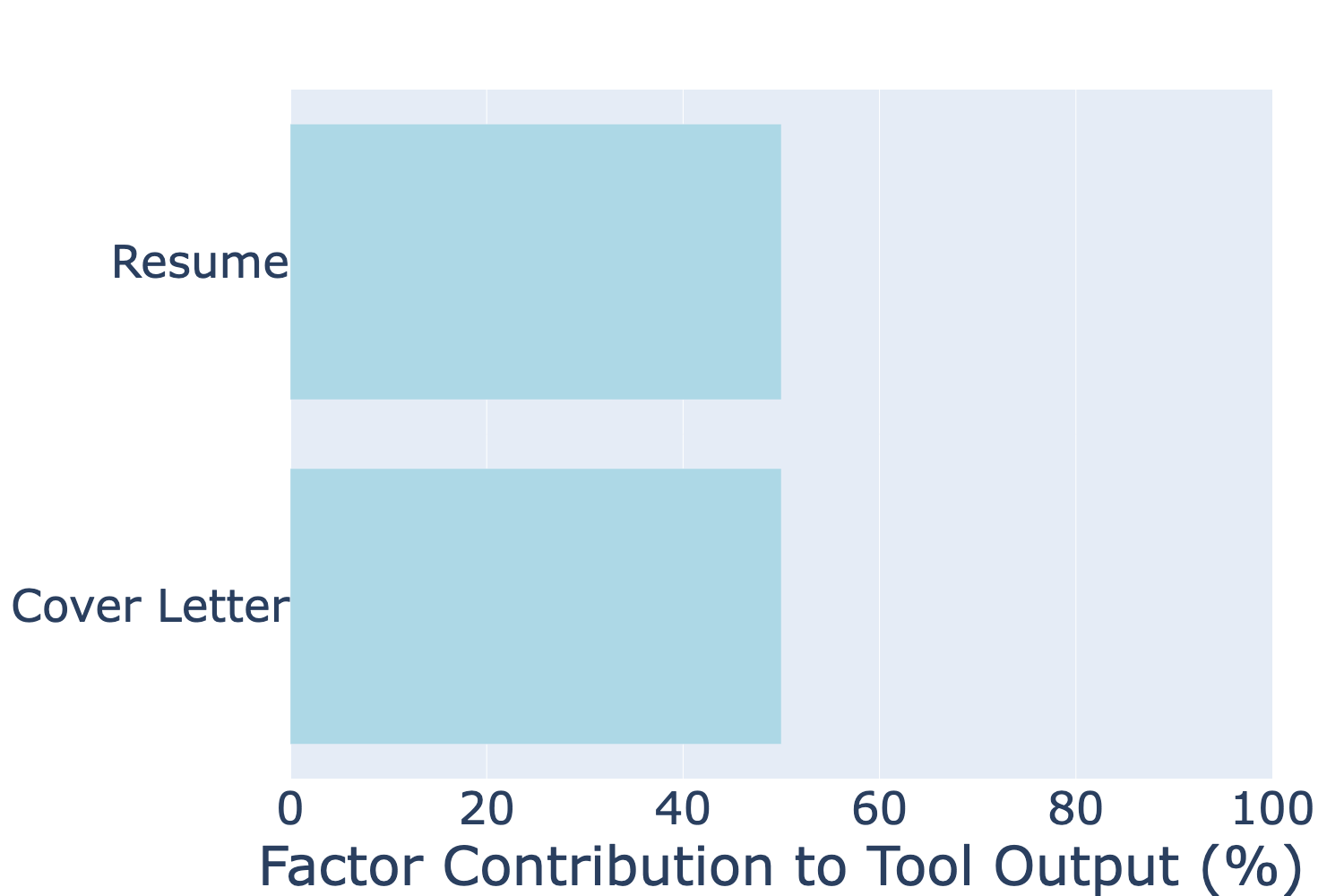}
    \hspace{.5em}\includegraphics[width=0.48\textwidth]{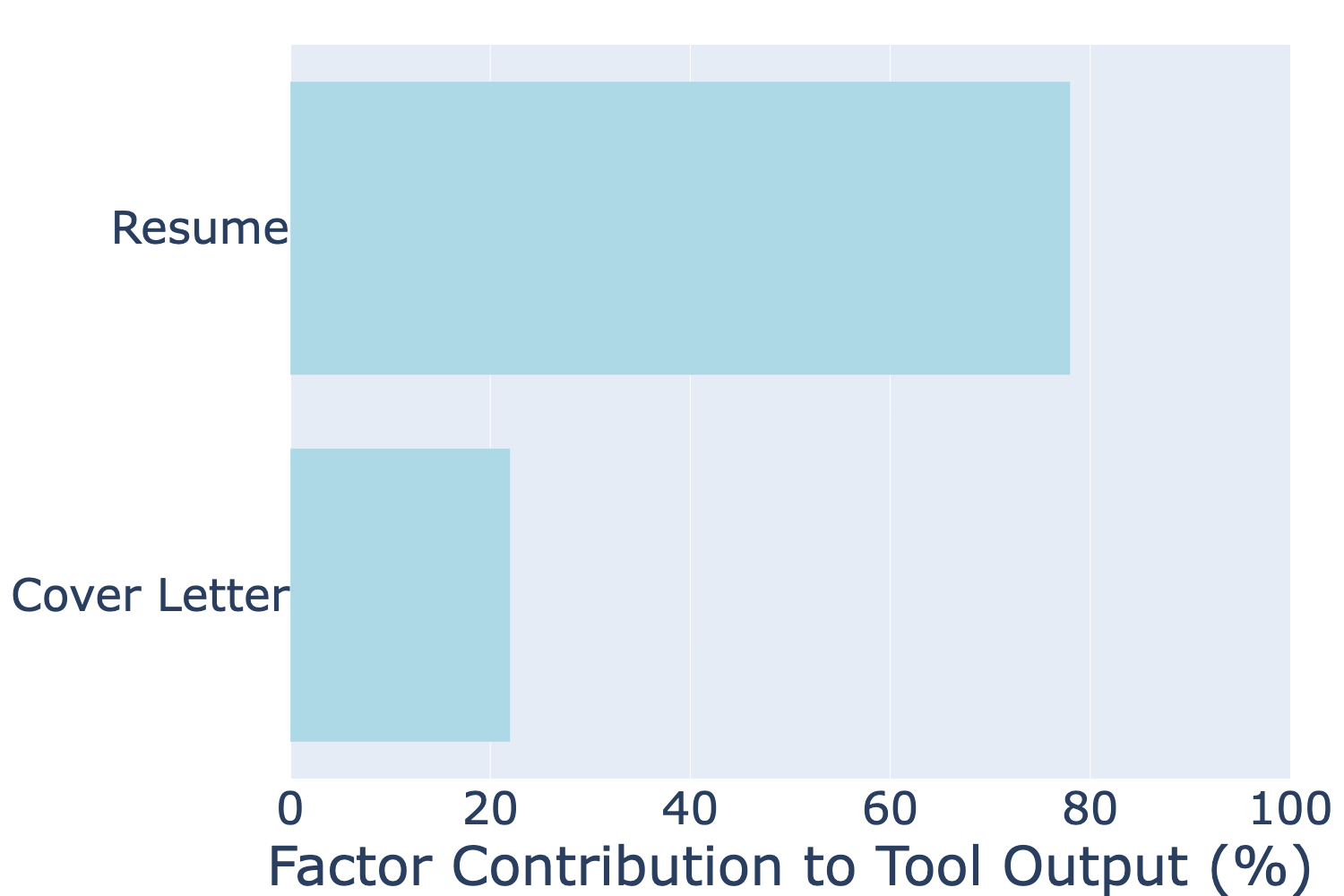}
    \includegraphics[width=0.48\textwidth]{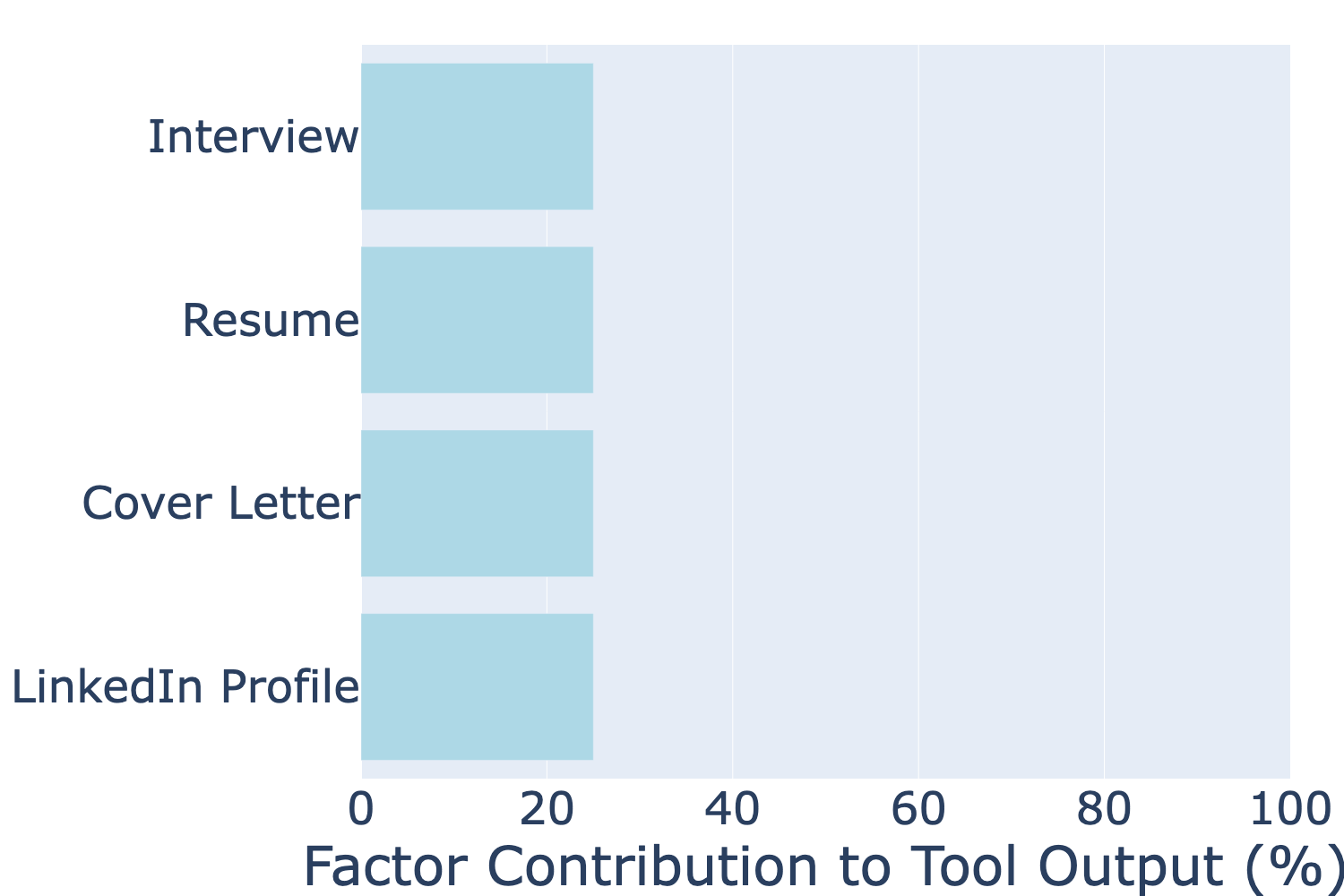}
    \includegraphics[width=0.48\textwidth]{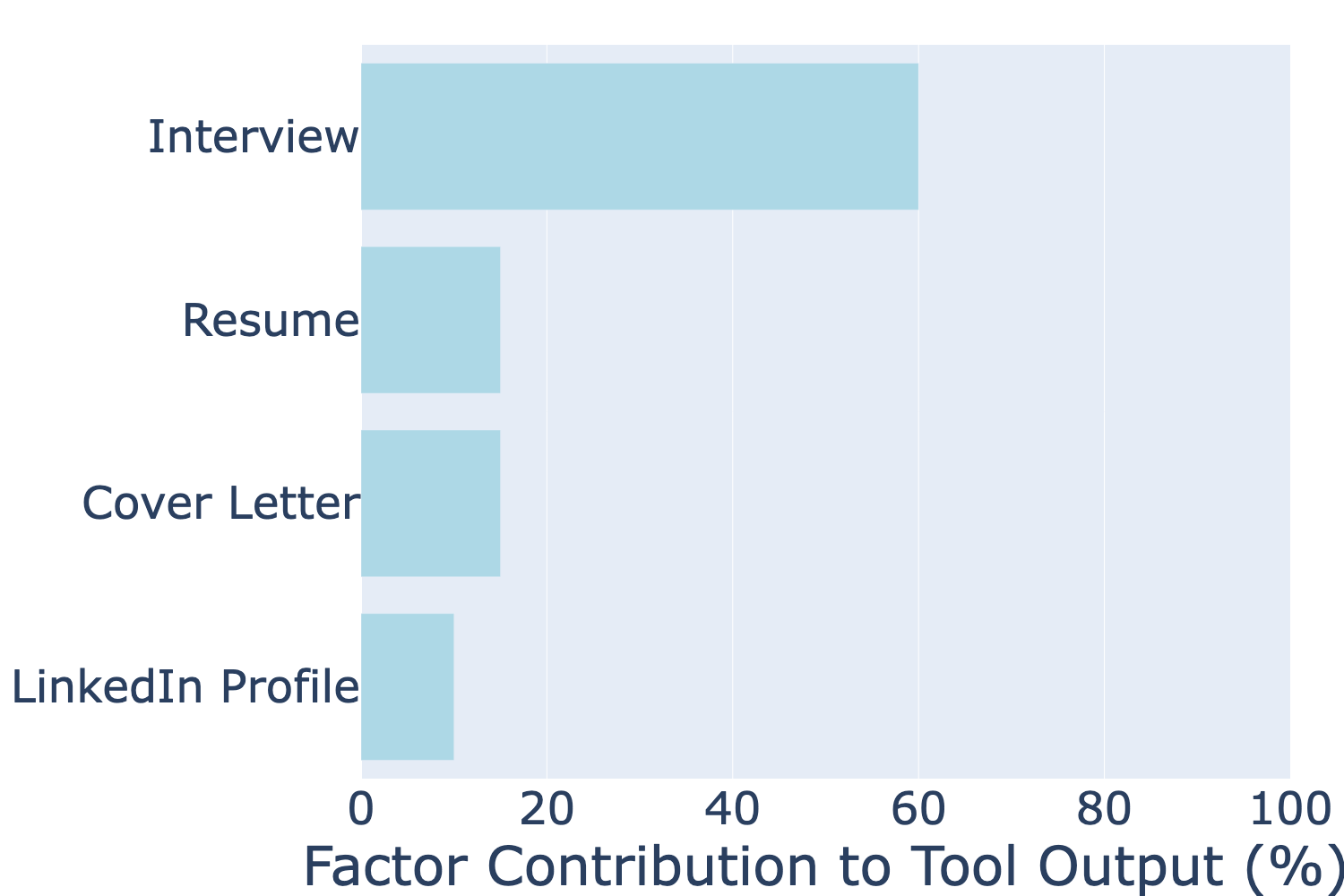}
    \caption{The scenarios shown to participants: two or four features of similar importance (top and bottom left resp.), and two or four features of differing importance (top and bottom right).}
    \label{fig:feature-based-explanations}
\end{figure}

\noindent\textbf{Task.} The task asked participants to give advice to a family member about how to prepare for the college application process. Participants were told an AI system provides predictions of college admissions. The features were resume (R) and cover letter (CL), with two additional features, interview (I) and LinkedIn profile (LP), for the four feature conditions. 

Each participant had a budget (10 hours) to allocate between the given features to improve the likelihood of acceptance recommendation. For starting features, in the two-feature and four-feature scenarios, we used $\vx_0=(40\text{ (R) }, 60\text{ (CL)})$ and $\vx_0=(60\text{ (I) }, 40\text{ (R) }, 60\text{ (CL) }, 65\text{ (LP)})$, where, the maximum for feature scores is 100. For each hour allocated to any feature, participants were given a piecewise linear cost: the feature's score improves by 5 points for the first four hours, 2.5 points for the second four hours, and 1 point for extra hours after that. 

\textbf{Correct (``optimal'') answers.} For two balanced features, one should \emph{not} invest more than 6 hours in any feature. In the scenario with two unbalanced features, the optimal investments in the resume and cover letter are 8 hours and 2 hours, respectively. For the scenario with four balanced features, the optimal is to invest at most 4 hours in any feature. For the unbalanced four features case, the optimal investment is to allocate 8 hours to the interview and the remaining 2 hours to the resume and cover letter. In this case, any investment in the LinkedIn profile feature is sub-optimal. A more detailed explanation is given in Appendix~\ref{sec:app-piece-wise-sol}. 

\noindent\textbf{Measures.} For other dependent measures, we used self-reported measurements of satisfaction, understanding, trust, and task performance using five-point semantic scales. 
We lightly edited questions from \citet{mohseni2021survey} for brevity and clarity.

\noindent\textbf{Recruitment.} We recruited 100 participants through Prolific in September 2024. 
Quotas on education level and gender ensured the sample was representative of the United States. Additionally, we gathered demographic information and assessed participants' familiarity with machine learning to ensure a representative and unbiased sample.

\subsection{Results and Discussion}\label{sec:user-study-results}
\textbf{Adding complexity reduces performance.} Our findings indicate that participant performance decreases when we increase the number of features from two to four or shift from balanced to unbalanced feature weights. We evaluate performance by comparing the total score of a response to the optimal total score for that case. The total score of each response is calculated by first determining the new feature vector $\vx$ by adding the improvements of each feature to $\vx_0$ based on the subject's response, and then calculating $\vtheta^T\vx$. 

\begin{figure}[ht]
    \centering
\includegraphics[width=0.28\textwidth]{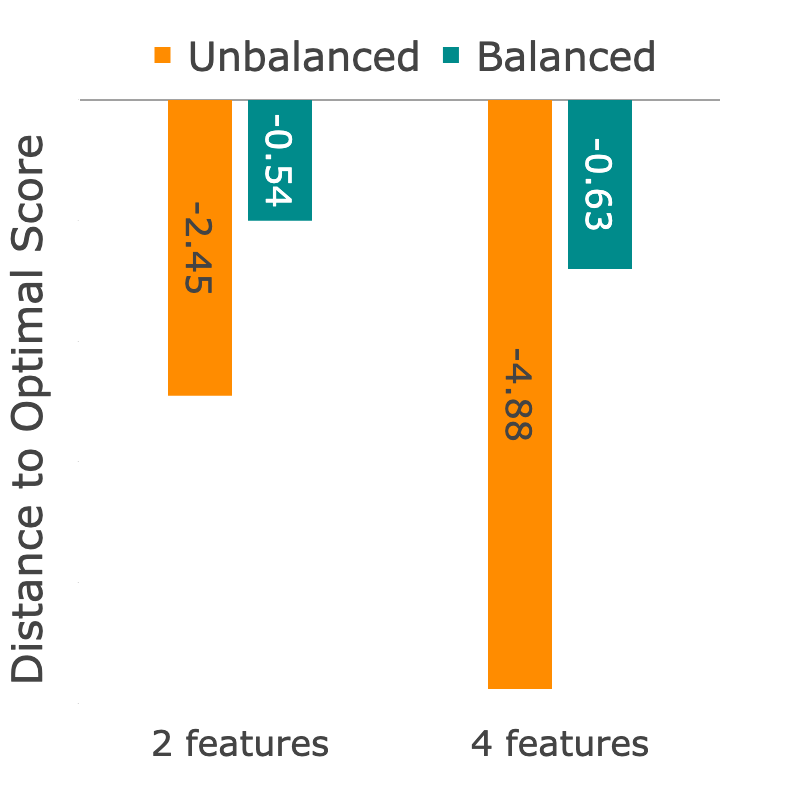}
    \caption{The average distance to optimal ($\vtheta^T\vx_\text{NB}-\vtheta^T\vx_\text{B}$) for the four scenarios.}
    \label{fig:dist-to-opt}
\end{figure}

In Figure~\ref{fig:dist-to-opt}, we observe that as the number of features increases, adding more complexity to the model, participants will move further away from the optimal score in balanced and unbalanced cases. The increase in the unbalanced case is almost double that of the balanced case. This indicates that we observe behavioral responses even in the balanced case, indicating biases beyond those in our theoretical predictions. For a fixed number of features, we see that answers are far from optimal when we have unbalanced features. Increasing the number of features will also lead to worse answers.

\begin{table}[ht]
    \caption{Number of responses in each scenario, (B) balanced and (U) unbalanced features} 
    \label{table:number-of-responses}
    \begin{center}
        \begin{tabular}{llll}
        \textbf{Scenario}  &\textbf{Opt.} &\textbf{1-feature} &\textbf{Sub-opt.} \\
        \hline \\
        2-features (B) & 21 & 0 & 6 \\
        4-features (B) & 14 & 1 & 10 \\
        2-features (U) & 5 & 3 & 16 \\
        4-features (U) & 1 & 1 & 22
        \end{tabular}
    \end{center}
\end{table}

Table~\ref{table:number-of-responses} shows that not only does the average distance from the optimal score increase with added complexity, but also the number of participants that responded sub-optimally increases. Most participants could find the optimal answers in the balanced cases, but most responded sub-optimally in the unbalanced cases. 

\textbf{Participants exhibit different behavioral biases.} The unbalanced scenarios indicate that most participants' behavior is consistent with following a Prelec function when (mis)perceiving feature importance. This leads them to under-invest in the most important feature and over-invest in the least important one. Participants not following the Prelec function tend to allocate all their budget to the most important feature.

The results from the balanced scenarios shed light on another behavioral bias: In the case of similar importance, participants invest more in the feature with a lower starting point (the resume). In the balanced scenarios, we notice that most participants respond without a behavioral bias, as predicted by the bias and Prelec functions. However, some participants responded sub-optimally, all over-investing in the resume. In the unbalanced four-feature case, the average investment in the resume is higher despite the resume and the cover letter having the same importance, indicating that this occurs for any two features with the same weight regardless of the importance of other features. 

Focusing on the unbalanced scenarios, we see that the number of participants who respond with the optimal answer drops when we increase the number of features. More participants decide to invest all their budget in the most important feature. Findings from the unbalanced two-feature scenario show that if participants do not invest all their budget in the most important feature, they under-invest in it. 

\begin{table}[ht]
    \caption{Average distance of investment in most important and least important features in unbalanced scenarios from the optimal.}\label{table:avg-dist-from-opt}
    \begin{center}
        \begin{tabular}{lll}
        \textbf{} &\textbf{Most important} &\textbf{Least important} \\
        \hline \\
        2-features & -2.13 hours & +2.13 hours \\
        4-features & -4.11 hours & +1.76 hours
        \end{tabular}
    \end{center}
\end{table}

Assuming the Prelec function, we find that $\gamma$ for the participants answering the unbalanced two-features scenario is $\gamma \le 0.64$, vs. $\gamma\le 0.55$ for four-features. (The lower the $\gamma$ in the Prelec function, the more intense the bias.) These upper bounds come from the fact that participants must underestimate the importance of the most important feature enough so that they conclude it is better to invest in the second most important feature.  

Another interesting observation in the unbalanced four-features scenario is that, even though any investment in the least important feature was sub-optimal, 18 participants still invested. This could be either a result of participants overestimating the importance of the least important feature, or a different behavioral bias, where participants prefer to invest in all options, and avoid leaving any feature as is. 
\section{Conclusion}
We present a strategic classification framework that accounts for the cognitive biases of strategic agents when assessing feature importance. We identify conditions under which the agents over- or under-invest in different features, the impacts of this on a firm's choice of classifier, and the impacts on the firm's utility and agents' welfare. Furthermore, through a user study, we support our theoretical model and results, showing that most participants respond sub-optimally when provided with an explanation of feature importance/contribution. 
{Exploring analytical models accounting for biases beyond misperception of feature weights, and exploring the possibility of designing explanation methods that can help mitigate biases, remain as important directions for further investigation.} 

\section*{Acknowledgment}
This work was supported in part by the UCSD Jacobs School Early Career Faculty Development Award. 

\bibliography{conference}

\clearpage

\appendix
\section{Additional related work}\label{sec:app-lit-review}
\textbf{Iterpretable Machine Learning and Explainable AI:} The interpretability of machine learning models and explainable AI is receiving more attention as it becomes more necessary for firms in various fields of work to explain an AI decision-making assistant or understand the algorithm's process for that specific decision \citep{ALI2023101805, peeking2018adadi}. Many studies have focused on providing guidelines and new objectives for the algorithm to ensure interpretability, \citet{freitas2014position} discusses the interpretability issues for five specific classification models and more generic issues of interpretability. \citet{lakkaraju2016decisionsets} provides a multi-objective optimization problem and uses model interpretability as a goal of the learning algorithm, \citet{lundberg2017shap} and \citet{ribeiro2016lime} provide methods for posthoc explanations. \citet{adebayo2018sanity} provides a method to evaluate an explanation method for image data. Many other works have also conducted user studies for finding or evaluating the guidelines using measures such as satisfaction, understanding, trust, etc. \citep{poursabzi2021manipulating, kulesza2015principles, sixt2022do}. 

\textbf{Strategic Classification:} Explanations enable users to potentially respond \citep{Camacho2011manipulation} to an algorithm to improve, meaning that they change their features to change their actual qualification or cheat, meaning that they manipulate their features to game the algorithm. This topic is extensively discussed in works such as \citet{Perdomo2020performative, Hardt2016strategic, Liu2020disparateequilibria, bechavod2021gaming, Hu2019disparate}. However, these works assume complete information on the model parameters, which is not necessarily a correct assumption. \citet{cohen2024bayesian} explores the partial information released by the firm and discusses the firm's optimization problem and agents' best response. \citet{haghtalab2023calibratedstackelberggameslearning} introduced the calibrated Stackelberg games where the agent does not have direct access to the firm's action. This can also be implemented in our framework where the firm uses $\vtheta$ but announces $\vtheta'$, and agents respond to $\vw(\vtheta')$. Another line of work called actionable recourse suggests giving actionable responses to users alongside the model explanation could be beneficial and help the users have a better outcome \citep{karimi2022recoursesurvey}. \citet{ustun2019recourse} provides an integer programming toolkit and introduces actionable recourse in linear regression. \citet{karimi2021algorithmicrecourse} introduces algorithmic recourse that allows people to act rather than understand. \citet{harris2022bayesian} combines the algorithmic recourse with partial information and has a firm that provides actionable recourse and steers agents. They show that agents and the firm are never worse off in expectation in this setting. 
\section{Proofs}\label{sec:app-proofs}

\paragraph{Proof of Lemma~\ref{lemma:band-optimization}, Lemma~\ref{lemma:quad-cost-band}, and Lemma~\ref{lemma:manhattan-cost-band}}We show the NB case, the B case can be shown similarly. We divide the agents into two subsets: (1) Agents that will attempt to optimize and (2) agents that will not attempt to optimize. The first subset is the agents that will have a non-negative utility after optimization, i.e., will have $r-c(\vx_\text{NB}, \vx_0)$. For these agents, since their reward is constant, the optimization problem comes down to:
\begin{align}
    &\vx_\text{NB} := \argmax_\vx ~ r - c(\vx, \vx_0) \notag\\
    &\text{subject to}\quad \vtheta^T\vx = \theta_0
\end{align}
And the agents that are in the second subset will solve $\vx_\text{NB} := \argmin_\vx ~ c(\vx, \vx_0)$ which is simply $\vx_\text{NB}=\vx_0$.

\textbf{Lemma~\ref{lemma:band-optimization}:} For norm-2 cost we know this is the same as finding the closest point on a hyperplane to a given point. We know the solution for this problem is to move in the direction of the normal vector of the hyperplane by $d(\vx_0, \vtheta, \theta_0)=\frac{\theta_0-\vtheta^T\vx_0}{\norm{\vtheta}_2}$. This means that the solution for the agents in the first subset is $\vx_\text{NB} = \vx_0 + d(\vx_0, \vtheta, \theta_{0})\vtheta$.

\textbf{Lemma~\ref{lemma:quad-cost-band}} The quadratic cost is similar to norm-2 cost, by directly solving the optimization problem and having $\lambda$ to be the Lagrange multiplier for the constraint we find:
\begin{align}
    x_{i, \text{NB}} = \frac{\lambda}{2}\frac{\theta_i}{c_i}+x_{i, 0} \text{ and } \frac{\lambda}{2} = \frac{\theta_0-\vtheta^T\vx_0}{\sum_j \frac{\theta_j^2}{c_j}}\Rightarrow x_{i, \text{NB}} = \frac{\theta_0-\vtheta^T\vx_0}{\sum_j\frac{\theta_j^2}{c_i}}\frac{\theta_i}{c_i}+x_{i, 0}
\end{align}
Which is, in some sense, a movement with a weighted distance from $\vx_0$ towards the hyperplane. 

\textbf{Lemma~\ref{lemma:manhattan-cost-band}}For the weighted Manhattan cost we are aiming to find the most efficient feature, i.e., the feature with the lowest $\frac{c_i}{\theta_i}$. 

\paragraph{Proof of Proposition~\ref{prop:under-invest-high-dim}}For a behavioral agent with $\vx_0$ that perceives $\theta_i$ as $\evw_i(\vtheta)$ to under-invest we need to have $\delta_i^{\text{B}}=d(\vx_0, \vw(\vtheta), \theta_0)\times \evw_i(\vtheta) < \delta_i^{\text{NB}}=d(\vx_0, \vtheta, \theta_0)\times \theta_i$, or $\frac{d(\vx_0, \vw(\vtheta), \theta_0)}{d(\vx_0, \vtheta, \theta_0)}<\frac{\theta_i}{\evw_i(\vtheta)}$. 

By knowing $\evw_i(\vtheta)<\theta_i$ then the agents with $d(\vx_0, \vw(\vtheta), \theta_0)\le d(\vx_0, \vtheta, \theta_0)$ will satisfy the condition since $\frac{d(\vx_0, \vw(\vtheta), \theta_0)}{d(\vx_0, \vtheta, \theta_0)}\le 1 < \frac{\theta_i}{\evw_i(\vtheta)}$ and under-invest in feature $i$. We can show the second statement similarly. 

The third statement of the proposition is a scenario where $\evw_1(\theta)<\theta_1$ where $\theta_1\ge \theta_i$ for all $i$, and we want to identify agents that will over-invest in that feature, i.e., $\frac{d(\vx_0, \vw(\vtheta), \theta_0)}{d(\vx_0, \vtheta, \theta_0)}>\frac{\theta_1}{\evw_1(\vtheta)}$. 

Since for the most important feature we have $\evw_1(\vtheta)=p(\theta_1)$, we can easily find the maximum of $\frac{\theta_1}{\evw_1(\vtheta)}$ for a given $\gamma$ by taking the derivative and using the function in \citet{Prelec1998}. This maximum occurs at $\theta^* = e^{-(\frac{1}{\gamma})^\frac{1}{\gamma-1}}$ meaning, $\frac{\theta_1}{\evw_1(\vtheta)}\le \frac{\theta^*}{\evw(\theta^*)} = \exp((\frac{1}{\gamma})^\frac{\gamma}{\gamma-1}-(\frac{1}{\gamma})^\frac{1}{\gamma-1})$. Therefore, using the same reasoning for the first two statements, agents with $\frac{d(\vx_0, \vw(\vtheta), \theta_0)}{d(\vx_0, \vtheta, \theta_0)}\ge \exp((\frac{1}{\gamma})^\frac{\gamma}{\gamma-1}-(\frac{1}{\gamma})^\frac{1}{\gamma-1})$ will over-invest in the most important feature, i.e., feature 1. 

\paragraph{Proof of Proposition~\ref{prop:mismatch-actual-b}}
We start the proof from the leftmost inequality in \eqref{eq:firm-loss-comp-benefit}. By the definition of $(\vtheta_\text{B}, \vtheta_{0, \text{B}})$ we can write $\E_{\vx\sim\mathcal{D}(\vw(\vtheta_\text{B}), \theta_{0, \text{B}})}[l(\vx, (\vtheta_\text{B}, \vtheta_{0, \text{B}}))]\le \E_{\vx\sim\mathcal{D}(\vw(\vtheta), \theta_{0})}[l(\vx, (\vtheta, \vtheta_0))]$ for all $(\vtheta, \vtheta_0)\neq (\vtheta_\text{B}, \vtheta_{0, \text{B}})$, i.e., $\sL((\vw(\vtheta_\text{B}), \theta_{0, \text{B}}), (\vtheta_\text{B}, \vtheta_{0, \text{B}}))\le \sL((\vw(\vtheta_\text{NB}), \theta_{0, \text{NB}}), (\vtheta_\text{NB}, \theta_{0, \text{NB}}))$ is always true. 

We next provide a characterization of the set of agents who fall within regions \framebox(7,9){1} and \framebox(7,9){3} in Figure~\ref{fig:highlighted}. These are the set of agents who will still pass the (true) decision boundary regardless of their biases. 
\begin{lemma}\label{lemma:H}
     For a given $(\vtheta, \theta_0)$, agents that satisfy $(1-\sigma(\vtheta))\theta_0\le(\vtheta-\sigma(\vtheta)\vw(\vtheta))^T\vx$, if given enough budget, will be accepted by the classifier, where $\sigma(\vtheta) \coloneqq \frac{\vtheta^T\vw(\vtheta)}{\norm{\vw(\vtheta)}^2}$ is a measure of the intensity of behavioral bias. 
\end{lemma}
\begin{proof}
    We can write agents' behavioral response as $\vx+\Delta_\text{B}$ with $\Delta_\text{B}=\frac{\theta_0-\vw(\vtheta)^T\vx}{\norm{\vw(\vtheta)}^2}\vw(\vtheta)$ for a given $(\vtheta, \theta_0)$. Agents that will have successful manipulation are the ones satisfying $\theta_0\le \vtheta^T(\vx+\Delta_\text{B})$ which, by substituting $\Delta_\text{B}$, can be written as:
\begin{align}
    &\vtheta_0\le \frac{\theta_0-\vw(\vtheta)^T\vx}{\norm{\vw(\vtheta)}^2}\vtheta^T\vw(\vtheta)+\vtheta^T\vx = \frac{\vtheta^T\vw(\vtheta)}{\norm{\vw(\vtheta)}^2}\theta_0+\bigg( \vtheta - \frac{\vtheta^T\vw(\vtheta)}{\norm{\vw(\vtheta)}^2} \vw(\vtheta) \bigg)^T \vx \notag \\
    &\Rightarrow(1-\sigma(\vtheta))\theta_0 \le (\vtheta-\sigma(\vtheta)\vw(\vtheta))^T\vx
\end{align}
    Where we defined $\sigma(\vtheta)\coloneqq\frac{\vtheta^T\vw(\vtheta)}{\norm{\vw(\vtheta}^2}$.
\end{proof}

To compare the firm's loss after biased and non-biased responses, we can break the feature space into the following regions ($\1(\cdot)$ is the indicator function):
\begin{enumerate}[label=\large\protect\textcircled{\small\arabic*}]
    \item $\1(\vtheta_\text{NB}^T\vx\ge\theta_{0, \text{NB}})$
    \item $\1(\vtheta_\text{NB}^T\vx\le\theta_{0, \text{NB}}-B)$
    \item $\1(\theta_{0, \text{NB}}-B\le\vtheta_\text{NB}^T\vx\le\theta_{0, \text{NB}})\1(\theta_{0, \text{NB}}-B\le\vw(\vtheta_\text{NB})^T\vx\le\theta_{0, \text{NB}}) \equiv \sA(\vtheta_\text{NB}, \theta_{0, \text{NB}})\cap \sA(\vw(\vtheta_\text{NB}), \theta_{0, \text{NB}})$
    \item $\1(\theta_{0, \text{NB}}-B\le\vtheta_\text{NB}^T\vx\le\theta_{0, \text{NB}})\1(\vw(\vtheta_\text{NB})^T\vx\ge\theta_{0, \text{NB}})$
    \item $\1(\theta_{0, \text{NB}}-B\le\vtheta_\text{NB}^T\vx\le\theta_{0, \text{NB}})\1(\vw(\vtheta_\text{NB})^T\vx\le\theta_{0, \text{NB}}-B)$
\end{enumerate}

We know that for $\vx\in {\Circled{1}}$ and $\vx\in\Circled{2}$, the expected loss for both response scenarios is the same since the agents in the two regions are either already qualified or will never make it to the decision boundary. Therefore, to compare the expected loss for two scenarios we would need to look at the differences in the rest of the regions. 

For $\vx\in\Circled{4}$ and $\vx\in\Circled{5}$ and biased responses, the expected loss would be the same as the non-strategic case. For $\vx\in\Circled{4}$ and $\vx\in\Circled{5}$ and the non-biased case, it could be higher or lower. For $\vx\in\Circled{3}$, the firm will have a lower (resp. higher) expected loss in the biased responses scenario if the truly unqualified agents are (resp. not) more than truly qualified agents. We furthermore focus on a subset of the region $\Circled{3}$ identified by Lemma~\ref{lemma:H}, region $\Circled{3a}$, which is the biased agents that will pass the threshold despite being biased. If we define the region identified by Lemma~\ref{lemma:H} by $\mathcal{H}(\vtheta_\text{NB}, \theta_{0, \text{NB}})$, then region $\Circled{3a}$ will be $\sA(\vtheta_\text{NB}, \theta_{0, \text{NB}})\cap \sA(\vw(\vtheta_\text{NB}), \theta_{0, \text{NB}}) \cap\mathcal{H}(\vtheta_\text{NB}, \theta_{0, \text{NB}})$. 

For a setting where the loss function rewards true positives and penalizes false positives as $-u^+ TP + u^- FP$, as higher loss is worse as we defined, we can write the following:
\begin{align}\label{eq:regions_NB}
    &\sL(\vtheta_\text{NB}, (\vtheta_\text{NB}, \theta_{0, \text{NB}}))=\mL_{\Circled{1}\cup\Circled{2}} + \int_{\vx\in\Circled{3}\cup\Circled{4}\cup\Circled{5}} \big( -u^+ p(\hat{y}=1 | \vx, y)f_1(\vx)\alpha_1 + u^- p(\hat{y}=1 | \vx, y)f_0(\vx)\alpha_0 \big) d\vx \\
    &\sL(\vw(\vtheta_\text{NB}), (\vtheta_\text{NB}, \theta_{0, \text{NB}}))=\mL_{\Circled{1}\cup\Circled{2}} + \int_{\vx\in\Circled{3a}} \big( -u^+ p(\hat{y}=1 | \vx, y)f_1(\vx)\alpha_1 + u^- p(\hat{y}=1 | \vx, y)f_0(\vx)\alpha_0 \big ) d\vx\label{eq:regions_B}
\end{align}

Where $\mL_{\Circled{1}\cup\Circled{2}}$ is the loss coming from regions $\Circled{1}$ and $\Circled{2}$ which is present in both scenarios. For $\sL(\vtheta_\text{NB}, (\vtheta_\text{NB}, \theta_{0, \text{NB}}))$, we know all the agents in $\Circled{3}\cup\Circled{4}\cup\Circled{5}$ will be accepted, i.e., $p(\hat{y}=1 | \vx\in\Circled{3}\cup\Circled{4}\cup\Circled{5}, y)=1$. Similar for $\sL(\vw(\vtheta_\text{NB}), (\vtheta_\text{NB}, \theta_{0, \text{NB}}))$ and $\vx\in\Circled{3a}$. 

We can see from \eqref{eq:regions_NB} and \eqref{eq:regions_B} that depending on the density of label 0 and label 1 agents in the region $\Circled{3a}$ and comparing it to the region $\Circled{3}\cup\Circled{4}\cup\Circled{5}$ we can have both $\sL(\vw(\vtheta_\text{NB}), (\vtheta_\text{NB}, \theta_{0, \text{NB}}))\le \sL(\vtheta_\text{NB}, (\vtheta_\text{NB}, \theta_{0, \text{NB}}))$ and $\sL(\vtheta_\text{NB}, (\vtheta_\text{NB}, \theta_{0, \text{NB}}))\le \sL(\vw(\vtheta_\text{NB}), (\vtheta_\text{NB}, \theta_{0, \text{NB}}))$ occur. The difference in expected loss lies in the region $\Circled{3}\cup\Circled{4}\cup\Circled{5}-\Circled{3a}$, or equivalently $\sS(\vtheta_\text{NB}, \theta_{0, \text{NB}}) \coloneqq \sA(\vtheta_\text{NB}, \theta_{0, \text{NB}})/(\sA(\vtheta_\text{NB}, \theta_{0, \text{NB}})\cap \sA(\vw(\vtheta_\text{NB}), \theta_{0, \text{NB}}) \cap\mathcal{H}(\vtheta_\text{NB}, \theta_{0, \text{NB}}))$, we can write the following for $\sL(\vtheta_\text{NB}, (\vtheta_\text{NB}, \theta_{0, \text{NB}}))-\sL(\vw(\vtheta_\text{NB}), (\vtheta_\text{NB}, \theta_{0, \text{NB}}))\le 0$ (resp. $\ge 0$):
\begin{align}
    \int_{\vx\in\sS(\vtheta_\text{NB}, \theta_{0, \text{NB}})}(-u^+f_1(\vx)\alpha_1+u^-f_0(\vx)\alpha_0)dx \le 0 \text{ (resp. $\ge$ 0)}
\end{align}

Therefore, if the density of unqualified agents is higher (resp.~lower) than the density of qualified agents over the region $\sA(\vtheta_\text{NB}, \theta_{0, \text{NB}})/(\sA(\vtheta_\text{NB}, \theta_{0, \text{NB}})\cap \sA(\vw(\vtheta_\text{NB}), \theta_{0, \text{NB}}) \cap\mathcal{H}(\vtheta_\text{NB}, \theta_{0, \text{NB}}))$, then:
\begin{align*}
    \sL(\vw(\vtheta_\text{NB}), (\vtheta_\text{NB}, \theta_{0, \text{NB}}))\le \sL(\vtheta_\text{NB}, (\vtheta_\text{NB}, \theta_{0, \text{NB}})) \quad (\text{resp. } \sL(\vtheta_\text{NB}, (\vtheta_\text{NB}, \theta_{0, \text{NB}}))\le \sL(\vw(\vtheta_\text{NB}), (\vtheta_\text{NB}, \theta_{0, \text{NB}})))
\end{align*}

To show the last statement of the proposition, we need to compare $\sL(\vw(\vtheta_\text{NB}), (\vtheta_\text{NB}, \theta_{0, \text{NB}}))$ and $\sL(\vw(\vtheta_\text{B}), (\vtheta_\text{B}, \theta_{0, \text{B}})))$ directly. The difference between these two losses comes from the region where agents will be accepted by $(\vtheta_\text{NB}, \theta_{0, \text{NB}})$ and not by $(\vtheta_\text{B}, \theta_{0, \text{B}})$, and vice versa, after agents' response. Mathematically, for agents responding to $(\vtheta_\text{NB}, \theta_{0, \text{NB}})$ without bias, we can show the agents accepted by $(\vtheta_\text{NB}, \theta_{0, \text{NB}})$ by $\sY(\vtheta_\text{NB}, \theta_{0,\text{NB}})\cup \sA(\vtheta_\text{NB}, \theta_{0,\text{NB}})$. We want the intersection of this set with the agents not accepted by $(\vtheta_\text{B}, \theta_{0, \text{B}})$, which brings us to $\sT_1=(\sY(\vtheta_\text{NB}, \theta_{0,\text{NB}})\cup \sA(\vtheta_\text{NB}, \theta_{0,\text{NB}}))\cap \sN(\vtheta_\text{B}, \theta_{0,\text{B}})$. 

Similarly, for agents responding to $(\vtheta_\text{NB}, \theta_{0, \text{NB}})$ with bias, we can show the agents accepted by $(\vtheta_\text{B}, \theta_{0, \text{B}})$ and not by $(\vtheta_\text{NB}, \theta_{0, \text{NB}})$ by $(\sY(\vtheta_\text{B}, \theta_{0,\text{B}}) \cap \sN(\vtheta_\text{NB}, \theta_{0,\text{NB}}))/\sA(\vtheta_\text{NB}, \theta_{0,\text{NB}})$. However, in this scenario, we need to also account for agents that make it past the actual decision boundary despite being behavioral, i.e., agents in the region $\mathcal{H}(\vtheta_\text{B}, \theta_{0,\text{B}})\cap \sA(\vw(\vtheta_\text{B}), \theta_{0,\text{B}})$, bringing us to $\sT_2 = (\mathcal{H}(\vtheta_\text{B}, \theta_{0,\text{B}})\cap \sA(\vw(\vtheta_\text{B}), \theta_{0,\text{B}}))\cup ( (\sY(\vtheta_\text{B}, \theta_{0,\text{B}}) \cap \sN(\vtheta_\text{NB}, \theta_{0,\text{NB}}))/\sA(\vtheta_\text{NB}, \theta_{0,\text{NB}}) )$. 

We need the total loss from region $\sT_1$ to be lower than the total loss from the region $\sT_2$ in the two scenarios for $\sL(\vtheta_\text{NB}, (\vtheta_\text{NB}, \theta_{0, \text{NB}}))\le \sL(\vw(\vtheta_\text{B}), (\vtheta_\text{B}, \theta_{0, \text{B}}))$ to be true. Meaning that we need $\int_{\vx\in\sT_1}(-u^+f_1(\vx)\alpha_1+u^-f_0(\vx)\alpha_0)d\vx \le \int_{\vx\in\sT_2}(-u^+f_1(\vx)\alpha_1+u^-f_0(\vx)\alpha_0)d\vx$ to be true for $\sL(\vtheta_\text{NB}, (\vtheta_\text{NB}, \theta_{0, \text{NB}}))\le \sL(\vw(\vtheta_\text{B}), (\vtheta_\text{B}, \theta_{0, \text{B}}))$, and the last inequality of the statement comes from the optimality condition. 
\section{Details of Numerical Experiments}\label{sec:app-numerical-details}
\paragraph{Details for Example~\ref{ex:firm-benefit-hurt} and Figure~\ref{fig:firm-benefit-hurt-dist}}For the scenario where the firm is negatively affected by the biased response is Example~\ref{ex:firm-benefit-hurt} we used $\vmu_1^T=(2, 4)$ and $\vmu_0^T=(2, 3)$ with $\Sigma_1=\begin{psmallmatrix}0.5 & 0 \\ 0 & 0.5 \end{psmallmatrix}$ and $\Sigma_0=\begin{psmallmatrix}1 & 0.5 \\ 0.5 & 1 \end{psmallmatrix}$, and we multiplied the generated data by 10. For the scenario where the firm benefits from agents' biased response we let $\vmu_1^T=(3, 5)$ and let the rest of the parameters be the same as the first scenario, i.e., $\vmu_0^T=(2, 3)$ with $\Sigma_1=\begin{psmallmatrix}0.5 & 0 \\ 0 & 0.5 \end{psmallmatrix}$ and $\Sigma_0=\begin{psmallmatrix}1 & 0.5 \\ 0.5 & 1 \end{psmallmatrix}$, and we multiplied the generated data by 10. In both scenarios, we let $B=5$. 

We used the Prelec function described in Section~\ref{sec:model} for the behavioral response. Solving the optimization problem takes a considerable amount of time for a large number of data points, here $20,000$, so we used the equivalent of the optimization problem for agents' movement and dictated the movement straight to each data point instead of solving the optimization.

To model agents' behavioral responses, we first identified the agents that would attempt to manipulate their features. Then, we used the movement function with the specified mode, either ``B'' or ``NB'', to move the data points and create a new dataset for post-response. 

For the last row of Figure~\ref{fig:firm-benefit-hurt-dist} we used $\vmu_1^T=(4, 4)$ and $\vmu_0^T=(2, 3)$ with $\Sigma_1=\begin{psmallmatrix}1 & 0 \\ 0 & 1 \end{psmallmatrix}$ and $\Sigma_0=\begin{psmallmatrix}3 & 0 \\ 0 & 1 \end{psmallmatrix}$, and we multiplied the generated data by 10. We used $B=10$.

\paragraph{Details for Figure~\ref{fig:BR-illustration}, Figure~\ref{fig:BR-illustration-quad-cost}, and Figure~\ref{fig:BR-illustration-lin-cost}} We generated 150 data points using different distributions for each feature. Feature 1 was sampled from $\mathcal{N}(700, 200)-\mathcal{D}((0, 20, 50, 100),(0.6, 0.2, 0.1, 0.1))$ where the second term is a discrete distribution selecting 0 with $p=0.6$, 20 with $p=0.2$, 50 with $0.1$, and 100 with $p=0.1$. Feature 2 was sampled from $1500-\Gamma(4, 100)$. We used a $Score$ column to label each individual for later. The score was calculated from the feature weights $(0.65, 0.35)$. We then used a sigmoid function to assign approval probability and label the sampled data points: $\frac{1}{1+\exp(-0.8\times (\frac{x}{10}-80))}$. We assigned the labels using the calculated approval probability and a random number generator. After generating the dataset, we used two copies, one for behavioral response and one for non-behavioral response. 

In Figure~\ref{fig:BR-illustration}, for agents' response to the algorithm, we calculated the agents that can afford the response with a budget of $B=100$ and performed an optimization problem on only those agents. We solved a cost minimization problem for each agent in the band specified by Lemma~\ref{lemma:band-optimization}: $\argmin_\vx cost=\norm{\vx-\vx_0}_2$ s.t. $\vtheta^T\vx\ge\theta_0$. For the behavioral case, we used $\gamma=0.5$, and the optimization problem $\argmin_\vx cost=\norm{\vx-\vx_0}_2$ s.t. $\vw(\vtheta)^T\vx\ge\theta_0$. 

In Figure~\ref{fig:BR-illustration-quad-cost}, for agents' response to the algorithm, we calculated the agents that can afford the response with a budget of $B=100$ and performed an optimization problem on only those agents. We solved a cost minimization problem for each agent in the band specified by Lemma~\ref{lemma:quad-cost-band}: $\argmin_\vx cost=\sum_i c_i(x_i-x_{0, i})^2$ s.t. $\vtheta^T\vx\ge\theta_0$. For the behavioral case, we used $\gamma=0.5$, and the optimization problem $\argmin_\vx cost=\sum_i c_i(x_i-x_{0, i})^2$ s.t. $\vw(\vtheta)^T\vx\ge\theta_0$. 

In Figure~\ref{fig:BR-illustration-lin-cost}, for agents' response to the algorithm, we calculated the agents that can afford the response with a budget of $B=100$ and performed an optimization problem on only those agents. We solved a cost minimization problem for each agent in the band specified by Lemma~\ref{lemma:manhattan-cost-band}: $\argmin_\vx cost=\vc^T|\vx-\vx_0|$ s.t. $\vtheta^T\vx\ge\theta_0$. For the behavioral case, we used $\gamma=0.5$, and the optimization problem $\argmin_\vx cost=\vc^T|\vx-\vx_0|$ s.t. $\vw(\vtheta)^T\vx\ge\theta_0$. 
\section{Agents' Welfare}\label{app:welfare}
Figure~\ref{fig:SW-regions} highlights the change in utility when agents are behaviorally biased (vs. when they were rational) across different regions in the feature space, with the regions generated based on the firm's optimal choice of threshold and agents' responses to it. In particular, the utility of agents in the green-highlighted region (this is $\sY(\vtheta_\text{B}, \theta_{0,\text{B}}) \cap \sN(\vtheta_\text{NB}, \theta_{0,\text{NB}})$ in Proposition~\ref{prop:mismatch-actual-b}) increases when they are behaviorally biased. One subset of agents in this region are those 
who in the rational case exert effort to get admitted and have a utility $r-c(\vx,\vx_0)$, whereas in the behaviorally biased case they attain utility $r > r-c(\vx,\vx_0)$ as they get admitted without any effort (and they correctly assume so). Another one is 
the subset of agents who would not try to get to the decision boundary in the rational case (and so have utility of $0$), but in the behavioral case, they are receiving utility $r$ without any movement and due to the change of the decision boundary. For the numerical example in the bottom row of Figure~\ref{fig:firm-benefit-hurt-dist}, there are more agents in this green-highlighted region than in the remaining red-highlighted regions (where biased agents have lower utility than rational agents), leading to an overall higher welfare for all agents when they are biased compared to when they were rational. 

\begin{figure}[ht]
    \centering
    \vspace{-0.1in}
    \includegraphics[width=0.4\linewidth]{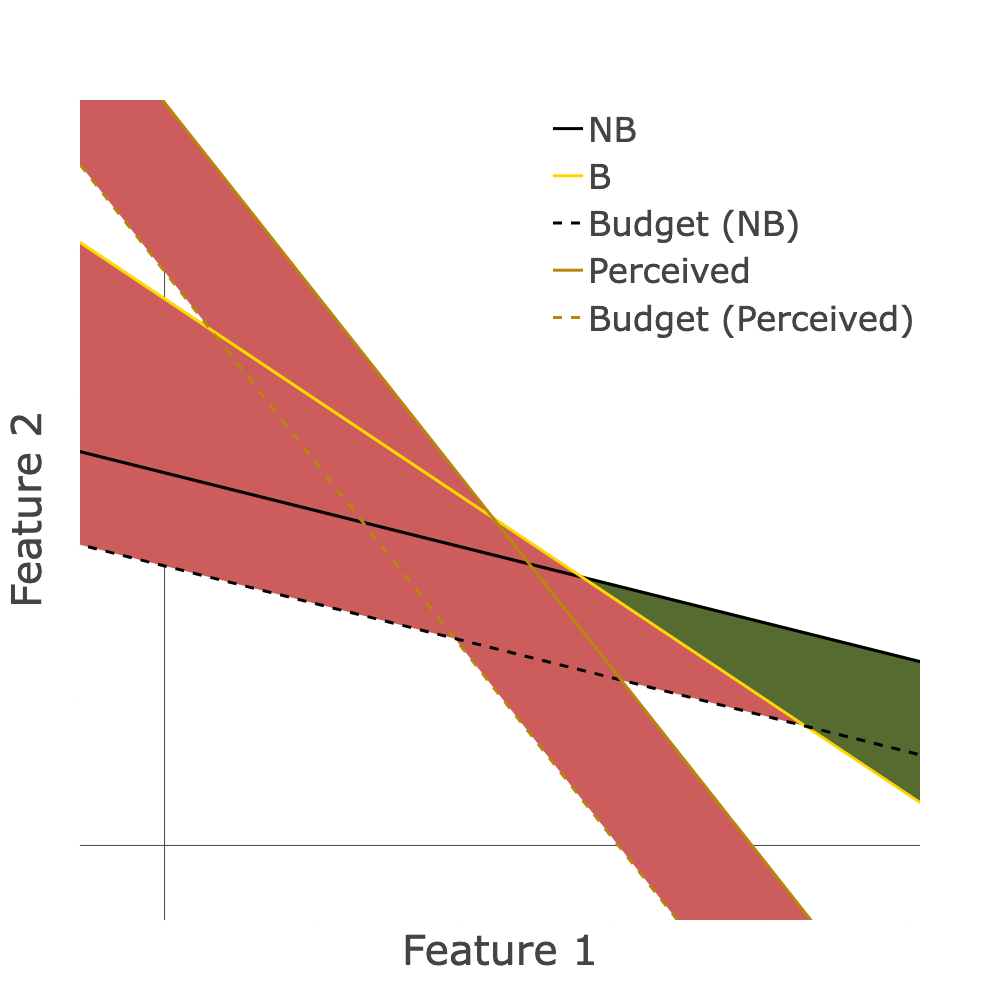}
    \caption{Regions where agents have higher (green) or lower (red) utility when biased vs. when rational.}
    \label{fig:SW-regions}
    \vspace{-0.1in}
\end{figure}
\section{Piece-wise Cost Function Solution}\label{sec:app-piece-wise-sol}
Consider a setting similar to the piece-wise cost function described. To decide the feature to spend $B_1$ of the budget, we are comparing $\frac{c_1}{\theta_1}$, $\frac{c_1}{\theta_2}$, and $\frac{c_1}{\theta_3}$ as they all have the same cost for the first step of the budget. Without loss of generality imagine we have $\frac{c_1}{\theta_1} < \frac{c_1}{\theta_2} < \frac{c_1}{\theta_3}$, therefore, we choose to allocate the $B_1$ amount of our budget to the first feature. For $B_2$, we do a similar comparison but we have to use $c_2$ for the first feature since the first feature is now in the second step, i.e., we compare $\frac{c_2}{\theta_1}$, $\frac{c_1}{\theta_2}$, and $\frac{c_1}{\theta_3}$. This could lead to resulting in investing in another feature, for example, if we have $\frac{c_1}{\theta_2} < \frac{c_2}{\theta_1} < \frac{c_1}{\theta_3}$, we would choose the second feature and invest $B_2$ in that feature. We continue this reasoning until we have reached the boundary. We designed our user study so the participants did not need to calculate if they reached the decision boundary and had all participants spend all their budgets. 

As seen in Figure~\ref{fig:2d-approx-illustration}, the quadratic cost movement differs from norm-2 movement, which moves the point to the closest point on the decision boundary. The piece-wise function we use for our user study is similar to a quadratic cost function with $c_2=0.85c_1$ and a decision boundary $0.78x_1+0.22x_2=70$ for the two-dimensional case. 


\vfill



\end{document}